\documentclass[10pt]{article}
\usepackage{amsmath,amsthm,amsfonts,amssymb,latexsym,graphicx}
\usepackage[noadjust]{cite}
\usepackage{algorithm}
\usepackage[noend]{algpseudocode}
\algrenewcommand\algorithmicthen{\relax}
\algrenewcommand\algorithmicdo{\relax}
\usepackage[pdfpagemode=UseNone,pdfstartview=FitH]{hyperref}

\allowdisplaybreaks[4]

\newtheorem{theorem}{Theorem}
\newtheorem{lemma}[theorem]{Lemma}

\theoremstyle{definition}

\newtheorem{remark}[theorem]{Remark}

\DeclareMathOperator{\Loss}{Loss}

\newcommand{\R}{\mathbb{R}}

\newcommand*{\ddd}{\mathrm{d}}

\title{Protected probabilistic classification}
\author{Vladimir Vovk, Ivan Petej, and Alex Gammerman}

\begin{document}
\maketitle

\begin{abstract}
  This paper proposes a way of protecting probabilistic prediction models against changes in the data distribution,
  concentrating on the case of classification and paying particular attention to binary classification.
  This is important in applications of machine learning,
  where the quality of a trained prediction algorithm may drop significantly in the process of its exploitation.
  Our techniques are based on recent work on conformal test martingales
  and older work on prediction with expert advice, namely tracking the best expert.

   The version of this paper at \url{http://alrw.net} (Working Paper 35)
   is updated most often and is accompanied with Python code\label{p:code}.
\end{abstract}

\section{Introduction}
\label{sec:introduction}

A common problem in applications of machine learning is that, soon after a prediction algorithm is trained,
the distribution of the data may change, and the prediction algorithm may need to be retrained.
There are efficient ways of online detection of a change in distribution,
such as using conformal test martingales \cite{Vovk:2021},
but there are inevitably awkward gaps between the change in distribution and its detection
and between the detection of the change and the deployment of a retrained prediction algorithm.

We will refer to the trained prediction algorithm as our \emph{prediction model}.
This paper proposes a way of preventing a catastrophic drop in the quality of the prediction model
when the data distribution changes.
Given a prediction model, our procedure gives a protected prediction model
that is more robust to changes in the data distribution.
To use Anscombe's \cite{Anscombe:1960} insurance metaphor
(repeatedly used already in \cite{Huber/Ronchetti:2009}),
our procedure provides an insurance policy against such changes.
The main desiderata for such a policy are its low price and its efficiency.

The case of regression was discussed in an earlier paper \cite{Vovk:arXiv2105}.
In this paper we concentrate on the simpler case of classification,
and first of all binary classification.
We will often assume that the label space is $\{0,1\}$.
Suppose we are given a predictive system that maps past data and an object $x$
with an unknown label $y\in\{0,1\}$
to a number $p\in[0,1]$, interpreted as the predicted probability that $y=1$.
We will refer to it as our \emph{base predictive system}.
In this paper we are mostly interested in the case where the base predictive system
is a prediction model obtained by training a prediction algorithm,
and so the predicted probability that $y=1$ depends only on the object $x$,
but allowing the dependence on the past data  does not complicate the exposition.

We will be interested in two seemingly different questions about the base predictive system:
\begin{description}
\item[Online testing]
  Can we gamble successfully against the base predictive system
  (at the odds determined by its predicted probabilities)?
  We are interested in online testing \cite{Vovk:2021},
  i.e., in constructing test martingales (nonnegative martingales with initial value 1)
  with respect to the base predictive system
  that take large values on the actual sequence of observations.
\item[Online prediction]
  Can we improve the base predictive system,
  modifying its predictions $p_n$ to better predictions $p'_n$?
\end{description}
If the quality of online prediction is measured using the log-loss function \cite{Good:1952},
the difference between the two questions almost disappears,
as we will see in Sections~\ref{sec:prediction} and~\ref{sec:experimental}.

After discussing online testing in Section~\ref{sec:testing} and online prediction in Section~\ref{sec:prediction},
we will give an example of a theoretical performance guarantee for our prediction procedure
(an application of a known technique) in Section~\ref{sec:theory}.
In Section~\ref{sec:experimental} we report encouraging experimental results,
and Section~\ref{sec:conclusion} concludes.

Two appendixes contain proofs and further experimental results.
Our computer code for the experiments is released in the form of Jupyter notebooks
(see p.~\pageref{p:code}, the end of abstract).

\section{Testing predictions by betting}
\label{sec:testing}

We consider a potentially infinite sequence of \emph{actual observations} $z_1,z_2,\dots$,
each consisting of two components: $z_n=(x_n,y_n)$,
where $x_n\in\mathbf{X}$ is an \emph{object} chosen from an \emph{object space} $\mathbf{X}$,
and $y_n\in\{0,1\}$ is a binary label.
A \emph{predictive system} is a function that maps any object $x$
and any finite sequence of observations $z_1,\dots,z_i$ (intuitively, the past data) for any $i\in\{0,1,\dots\}$
to a number $p\in[0,1]$ (intuitively, the probability that the label of $x$ is 1).
Fix a \emph{base predictive system},
and let $p_1,p_2,\dots$ be its predictions for the actual observations:
$p_n$ is the prediction output by the base predictive system on $x_n$ and $z_1,\dots,z_{n-1}$;
it is interpreted as the predicted probability that $y_n=1$.
(We will not need any measurability assumptions;
in particular, $\mathbf{X}$ is not supposed to be a measurable space.)

In this paper we are mostly interested in the special case where the output $p$ of the base predictive system
depends only on $x$ and not on $z_1,\dots,z_i$.
In this case we will say that our predictive system is a \emph{prediction model}.
A typical way in which prediction models appear in machine learning is as result of training a prediction algorithm.
In Section~\ref{sec:experimental} we will only consider prediction models,
but for now we do not impose this restriction.

\begin{algorithm}[bt]
  \caption{Simple Jumper martingale ($(p_1,p_2,\dots)\mapsto(S_1,S_2,\dots)$)}
  \label{alg:SJ-test}
  \begin{algorithmic}[1]
    \State $C_{\theta}:=1_{\theta=\boldsymbol{0}}$ for all $\theta\in\Theta$\label{l:prior}
    \State $C:=1$
    \For{$n=1,2,\dots$:}
      \For{$\theta\in\Theta$:}
        $C_{\theta} := (1-J)C_{\theta} + (J/\left|\Theta\right|)C$\label{l:C-epsilon}
      \EndFor
      \For{$\theta\in\Theta$:}
          $C_{\theta} := C_{\theta} B_{f_{\theta}(p_n)}(\{y_n\}) / B_{p_n}(\{y_n\})$
      \EndFor
      \State $S_n := C := \sum_{\theta\in\Theta} C_{\theta}$
    \EndFor
  \end{algorithmic}
\end{algorithm}

Our first online testing procedure is given as Algorithm~\ref{alg:SJ-test},
where we use the notation $1_E$ to mean 1 if a property $E$ holds and 0 if not.
One of its two parameters is a finite non-empty family $f_{\theta}:[0,1]\to[0,1]$, $\theta\in\Theta$,
of \emph{calibrating functions}.
The intuition behind $f_{\theta}$ is that we are trying to improve the base predictions $p_n$,
or \emph{calibrate} them;
the idea is to use a new prediction $f_{\theta}(p_n)$ instead of $p_n$.
We assume that $\Theta$ contains a distinguished element $\boldsymbol{0}\in\Theta$
(used in line~\ref{l:prior} of Algorithm~\ref{alg:SJ-test}).

Perhaps the simplest choice (explored in Appendix~\ref{app:further})
is to use a finite subset of the family
\begin{equation}\label{eq:quadratic}
  f_{\theta}(p) := p + \theta p(1-p),
\end{equation}
where $\theta\in[-1,1]$, and $\theta=0$ is the distinguished element.
For $\theta>0$ we are correcting for the forecasts $p$ being underestimates of the true probability of 1,
while for $\theta<0$ we are correcting for $p$ being overestimates;
$f_0$ is the identity function (no correction).

We do not know in advance which $f_{\theta}$ will work best,
and moreover, it seems plausible that suitable values of $\theta$ will change over time.
Therefore, we use the idea of ``tracking the best expert'' \cite{Herbster/Warmuth:1998ML}.
Algorithm~\ref{alg:SJ-test} uses the notation $B_p$, $p\in[0,1]$,
for the Bernoulli distribution on $\{0,1\}$ with parameter $p$:
$B_p(\{1\})=p$.
To each sequence $\theta=(\theta_1,\theta_2,\dots)$ of elements of $\Theta$ corresponds
the \emph{elementary test martingale}
\begin{equation}\label{eq:el-test}
  \prod_{i=1}^n
  \frac{B_{f_{\theta_i}(p_i)}(\{y_i\})}{B_{p_i}(\{y_i\})},
  \quad
  n=0,1,\dots.
\end{equation}
The other parameter of the Simple Jumper martingale of Algorithm~\ref{alg:SJ-test} is $J\in[0,1]$, the \emph{jumping rate}.
This martingale is obtained by ``derandomizing'' (to use the terminology of \cite{Vovk:1999derandomizing}) the stochastic test martingale
corresponding to the probability measure $\mu$ on $(\theta_1,\theta_2,\dots)\in[0,1]^{\infty}$
defined as the probability distribution of the following Markov chain with state space $\Theta$.
The initial state $\theta_0$ (ignored by $\mu$) is $\boldsymbol{0}$ (line \ref{l:prior} of Algorithm~\ref{alg:SJ-test}),
and the transition function prescribes maintaining the same state with probability $1-J$
and, with probability $J$, choosing a new state from the uniform probability measure on $\Theta$
(line~\ref{l:C-epsilon}).
We will sometimes refer to it as the \emph{Simple Jumper Markov chain}.

We derandomize the stochastic test martingale by averaging,
\begin{equation*}
  S_n
  :=
  \int
  \prod_{i=1}^n
  \frac{B_{f_{\theta_i}(p_i)}(\{y_i\})}{B_{p_i}(\{y_i\})}
  \mu(\ddd\theta),
\end{equation*}
which gives us a deterministic test martingale.
This construction is standard in conformal testing \cite[Section 3]{Vovk:2021}.

In one respect the Simple Jumper martingale is not adaptive enough:
we assume that a suitable jumping rate $J$ is known in advance.
Instead, we will use the \emph{Composite Jumper} martingale
that depends on two parameters,
$\pi\in(0,1)$ and a finite set $\mathbf{J}$ of non-zero jumping rates.
It is defined to be the weighted average
\begin{equation}\label{eq:average}
  S_n
  :=
  \pi
  +
  \frac{1-\pi}{\left|\mathbf{J}\right|}
  \sum_{J\in\mathbf{J}}
  S^J_n,
\end{equation}
where $S^J$ is computed by Algorithm~\ref{alg:SJ-test} fed with $J$ as its parameter.

\begin{figure}
  \begin{center}
    \includegraphics[width=0.6\textwidth]{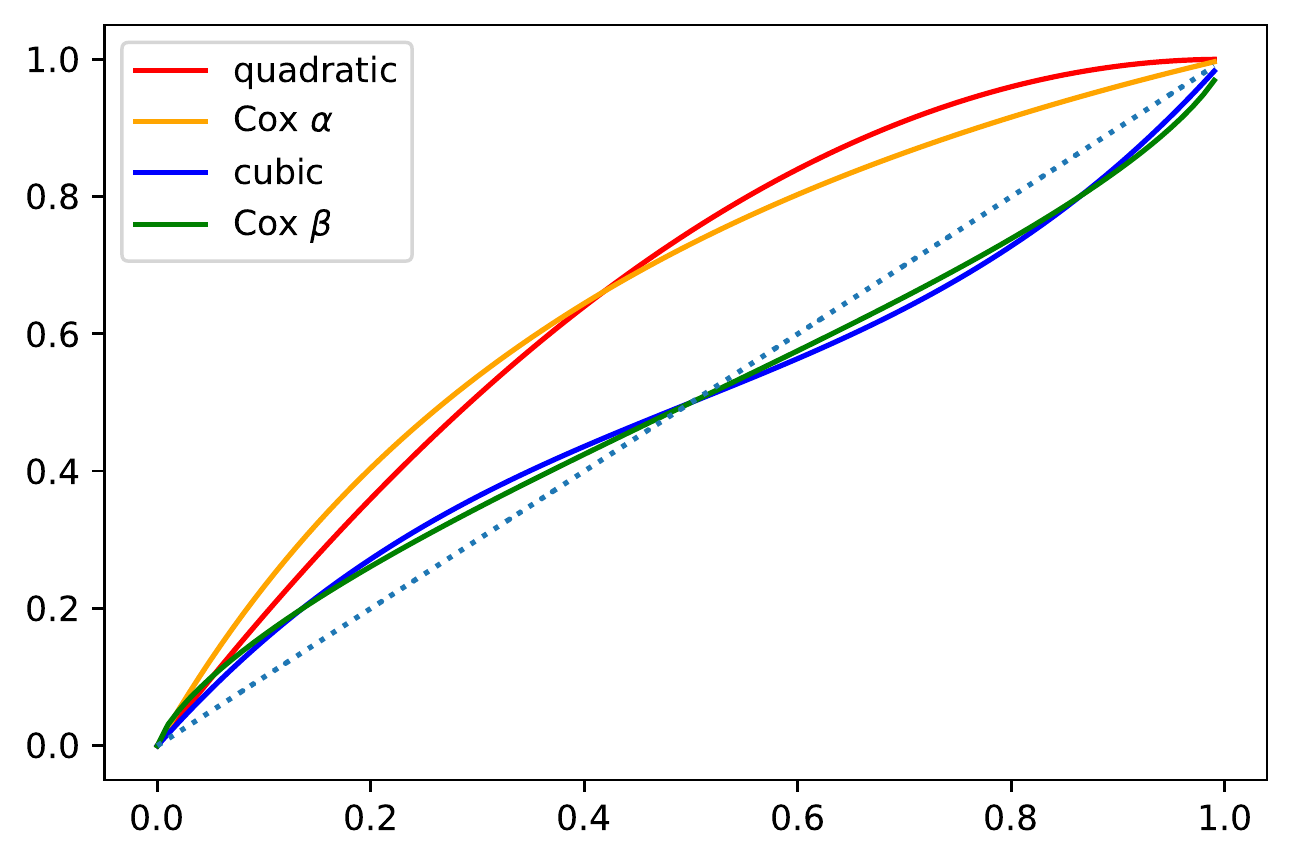}
  \end{center}
  \caption{Examples of calibration functions.}
  \label{fig:cal_functions}
\end{figure}

In Appendix~\ref{app:further} we will see that already the simple choice \eqref{eq:quadratic}
leads to very successful betting for popular benchmark datasets.
However, there are numerous other natural calibration functions,
some of which are shown in Figure~\ref{fig:cal_functions}.
The function in red is in the quadratic family \eqref{eq:quadratic} (and its parameter is $\theta=1$);
these functions are fully above, fully below, or (for $\theta=0$) situated on the bisector of the first quadrant
(shown as the dotted line).
In many situations other calibration functions will be more suitable.
For example, it is well known that untrained humans tend to be overconfident
\cite[Part VI, especially Chapter 22]{Kahneman/etal:1982}.
A possible calibration function correcting for overconfidence is the cubic function
\begin{equation}\label{eq:cubic}
  f_{a,b}(p) := p + a p(p-b)(p-1),
\end{equation}
where $(a,b)=\theta\in[0,1]^2$.
An example of such a function is shown in Figure~\ref{fig:cal_functions} in blue
(with parameters $a=1.5$ and $b=0.5$).
The meaning of the parameters is that $b$ is the value of $p$ (such as $0.5$) that we believe does not need correction,
and that $a$ indicates how aggressively we want to correct for overconfidence
($a<0$ meaning that in fact we are correcting for underconfidence).
If the predictor predicts a $p$ that is close to 0 or 1,
we correct for his overconfidence (assuming $a>0$) by moving $p$ towards the neutral value $b$.

In the experimental Section~\ref{sec:experimental}
we will use Cox's \cite[Section 3]{Cox:1958two} calibration functions
\begin{equation}\label{eq:Cox}
  f_{\alpha,\beta}(p)
  :=
  \frac{p^{\beta}\exp(\alpha)}{p^{\beta}\exp(\alpha)+(1-p)^{\beta}},
\end{equation}
where $(\alpha,\beta)=\theta\in\R^2$.
We obtain the identity function $f_{\alpha,\beta}(p)=p$ for $\alpha=0$ and $\beta=1$;
these are the ``neutral'' values, $\boldsymbol{0}=(0,1)\in\Theta$.
Cox starts his exposition in \cite[Section 3]{Cox:1958two} from the one-parameter subfamily
\begin{equation}\label{eq:Cox-beta}
  f_{\beta}(p)
  :=
  \frac{p^{\beta}}{p^{\beta}+(1-p)^{\beta}},
\end{equation}
where $\beta\in\R$,
obtained by fixing $\alpha:=0$.
We are mostly interested in $\beta>0$, although $\beta=0$ (when every $p$ is transformed to 0.5)
and $\beta<0$ (which reverses the order of the label probabilities) are also possible.
Similarly, we can set $\beta$ to its neutral value 1 obtaining another one-parameter subfamily,
\begin{equation}\label{eq:Cox-alpha}
  f_{\alpha}(p)
  :=
  \frac{p\exp(\alpha)}{p\exp(\alpha)+(1-p)}.
\end{equation}
An example of a function in the class~\eqref{eq:Cox-beta} is shown in Figure~\ref{fig:cal_functions} in green (with $\beta=0.75$);
the plot in orange shows a function in the class~\eqref{eq:Cox-alpha} (with $\alpha=1$).

\begin{remark}
  Cox's calibration functions look particularly natural (are linear functions)
  in terms of the log odds ratios (as presented in Cox \cite[Section 3]{Cox:1958two}).
  Namely, \eqref{eq:Cox} can be rewritten as
  \begin{equation}\label{eq:Cox-odds-ratio}
    \log\frac{f_{\alpha,\beta}(p)}{1-f_{\alpha,\beta}(p)}
    :=
    \alpha
    +
    \beta\log\frac{p}{1-p}.
  \end{equation}
\end{remark}

\subsection*{Multiclass case}

The advantage of the representations~\eqref{eq:Cox}, \eqref{eq:Cox-beta}, and~\eqref{eq:Cox-alpha}
over \eqref{eq:Cox-odds-ratio}
is that they are easy to extend to the multiclass case.
Let the label space be $\mathbf{Y}$ with $K:=\left|\mathbf{Y}\right|<\infty$;
therefore, each prediction is a set of nonnegative numbers $\mathbf{p}=(p_y\mid y\in\mathbf{Y})\in[0,1]^K$ summing to 1,
where $p_y$ is the predicted probability of $y$.
The calibration functions~\eqref{eq:Cox} become
\begin{equation}\label{eq:Cox-multiclass}
  f_{\alpha,\beta}(\mathbf{p})_y
  :=
  \frac{p^{\beta}_y\exp(\alpha(y))}{\sum_{y'\in\mathbf{Y}}p_{y'}^{\beta}\exp(\alpha(y'))},
\end{equation}
where $\alpha\in\R^{\mathbf{Y}}$ and $\beta\in\R$.
Formally, the number of parameters in \eqref{eq:Cox-multiclass} is $K+1$,
but the effective number of parameters is $K$
since the transformation~\eqref{eq:Cox-multiclass} does not change
when the same constant (positive or negative) is added to all $\alpha(y)$, $y\in\mathbf{Y}$.
The calibration functions~\eqref{eq:Cox-beta} become
\begin{equation*}
  f_{\beta}(\mathbf{p})_y
  :=
  \frac{p_y^{\beta}}{\sum_{y'\in\mathbf{Y}} p_{y'}^{\beta}},
\end{equation*}
where $\beta\in\R$.
Therefore, this family still depends on one real-valued parameter, $\beta$.

\section{Protecting prediction algorithms}
\label{sec:prediction}

For any predictive system,
we define its \emph{probability process} as a function mapping any finite sequence of observations
to the product $B_{p_1}(y_1)\cdots B_{p_n}(y_n)$
(i.e., the probability attached to this sequence by the predictive system),
where $n$ is the number of observations in the sequence, $y_1,\dots,y_n$ are their labels,
and $p_1,\dots,p_n$ are the predictions for those observations.
We regard the probability process as the capital process of a player playing an extremely challenging (definitely unfair) game:
his capital cannot go up, and for it not to go down he has to predict
with the probability measure concentrated on the true outcome.
The probability process of the base predictive system will be referred to as the \emph{base probability process}.

\begin{remark}
  The notion of a probability process is very similar to Cox's \cite{Cox:1975partial} notion of partial likelihood,
  but we cannot say that a probability process is partial in any sense (since it is not part of a fuller probability process:
  there is no probability measure on the objects \cite[Section 10.5]{Shafer/Vovk:2019}).
  In the absence of objects it becomes close to likelihood
  (however, unlike likelihood, it is not a function of any parameters).
  It is even closer to the notion of measure as used in the algorithmic theory of randomness:
  see, e.g., \cite[Definition 4.2.1]{Li/Vitanyi:2019}.
\end{remark}

\begin{algorithm}[bt]
  \caption{Composite Jumper predictor ($(p_1,p_2,\dots)\mapsto(p'_1,p'_2,\dots)$)}
  \label{alg:CJ-predictor}
  \begin{algorithmic}[1]
    \State $P:=\pi$\label{ln:prior-1}
    \State $A^J_{\theta}:=\frac{1-\pi}{\left|\mathbf{J}\right|}1_{\theta=0}$ for all $J\in\mathbf{J}$ and $\theta\in\Theta$\label{ln:prior-2}
    \For{$n=1,2,\dots$:}
      \For{$J\in\mathbf{J}$:}
        \State $A:=\sum_{\theta}A^J_{\theta}$
        \For{$\theta\in\Theta$:}
          $A^J_{\theta} := (1-J)A^J_{\theta} + A J/\left|\Theta\right|$\label{ln:recursion-1}
        \EndFor
      \EndFor
      \State $p'_n := p_n P + \sum_{J,\theta} f_{\theta}(p_n) A^J_{\theta}$\label{ln:result}
      \State $P := P B_{p_n}(\{y_n\})$\label{ln:recursion-2}
      \For{$J\in\mathbf{J}$ and $\theta\in\Theta$:}
        $A^J_{\theta} := A^J_{\theta} B_{f_{\theta}(p_n)}(\{y_n\})$\label{ln:recursion-3}
      \EndFor
      \State $C := P + \sum_{J,\theta} A^J_{\theta}$\label{ln:C}
      \State $P := P/C$\label{ln:recursion-4}
      \For{$J\in\mathbf{J}$ and $\theta\in\Theta$:}
        $A^J_{\theta} := A^J_{\theta} / C$\label{ln:recursion-5}
      \EndFor
    \EndFor
  \end{algorithmic}
\end{algorithm}

For simplicity, we will discuss only positive probability processes and martingales
(i.e., those that do not take zero values).
This will be sufficient for the considerations of Section~\ref{sec:experimental}.
Each test martingale with respect to the base predictive system
is the ratio of a probability process to the base probability process
(i.e., is a probability ratio process, familiar from sequential analysis),
and vice versa.
This establishes a bijection between test martingales and probability processes
(for a fixed base predictive system).
Algorithm~\ref{alg:CJ-predictor} is the predictive system
whose probability process corresponds to the test martingale of Algorithm~\ref{alg:SJ-test}
averaged as in \eqref{eq:average}.

Algorithm~\ref{alg:CJ-predictor} implements the Bayesian merging rule
and is a special case the Aggregating Algorithm (AA) \cite{Vovk:1999derandomizing}
corresponding to the \emph{log-loss function}
\begin{equation}\label{eq:log-loss}
  \lambda(y,p)
  :=
  \begin{cases}
    -\log p & \text{if $y=1$}\\
    -\log(1-p) & \text{if $y=0$},
  \end{cases}
\end{equation}
where $y\in\{0,1\}$ is the true label and $p\in[0,1]$ is its prediction
(the logarithm is typically natural,
but in Section~\ref{sec:experimental} we will consider decimal logarithms).
In the case where $\pi=0$ and $\left|\mathbf{J}\right|=1$,
it is also a special case of the Fixed Share algorithm of \cite{Herbster/Warmuth:1998ML}.

The AA is described in \cite[Section 2]{Vovk:1999derandomizing},
and in our case of the log-loss function
the optimal in a natural sense learning rate is $\eta:=1$,
and the AA coincides with the APA
(``Aggregating Pseudo-Algorithm'').
Analogously to \eqref{eq:el-test} but reflecting the operation of averaging \eqref{eq:average},
to each $J\in\mathbf{J}$ and each sequence $\boldsymbol{\theta}=(\theta_1,\theta_2,\dots)$
corresponds the \emph{elementary predictor} that outputs, at each step $n$,
\begin{equation*}
  p'_n
  :=
  f_{\theta_n}(p_n),
  \quad
  n=1,2,\dots,
\end{equation*}
as its prediction.
There is also the \emph{base} elementary predictor that just coincides with the base predictive system.
The prior probability measure on the elementary predictors is built on top of the Simple Jumper Markov chain
(described in the previous section) and taking the averaging \eqref{eq:average} into account:
\begin{itemize}
\item
  With probability $\pi$, we choose the base elementary predictor
  (which is our \emph{passive} elementary predictor).
\item
  Otherwise, we choose the jumping rate $J$ from the uniform probability measure on $\mathbf{J}$.
\item
  Having chosen $J$, we generate $\boldsymbol{\theta}:=(\theta_1,\theta_2,\dots)$ as in the previous section,
  which gives us the \emph{active} elementary predictor $(J,\boldsymbol{\theta})$.
\end{itemize}
This determines the prior distribution on the elementary predictors.

The variable $P$ in Algorithm~\ref{alg:CJ-predictor} holds the posterior weight of the passive elementary predictor,
and $A^J_{\theta}$ holds the total posterior weight of the elementary predictors $(J,\boldsymbol{\theta})$
that are in the state $\theta$
(i.e., $\theta_n=\theta$, where $\boldsymbol{\theta}=(\theta_1,\theta_2,\dots)$
and $n$ is the current step).
We initialize them to their prior weights (lines~\ref{ln:prior-1}--\ref{ln:prior-2}),
and the recursion is given
in lines~\ref{ln:recursion-1}, \ref{ln:recursion-2}, \ref{ln:recursion-3}, \ref{ln:recursion-4}, and~\ref{ln:recursion-5}.
Line~\ref{ln:recursion-1} corresponds to the transition function of the Simple Jumper Markov chain with the jumping rate $J$,
and lines~\ref{ln:recursion-2}--\ref{ln:recursion-3} are the Bayesian weight updates.
In line~\ref{ln:C} we compute the total posterior weight of the elementary predictors,
and in lines~\ref{ln:recursion-4}--\ref{ln:recursion-5} we normalize the weights.
In line~\ref{ln:result} we compute the protected prediction
as weighted average of the predictions produced by the elementary predictors.

We refer to the method exemplified by Algorithm~\ref{alg:CJ-predictor}
as \emph{protected probabilistic classification}.
We start from a base predictive system,
design a way of gambling against it (a test martingale),
and then ``turn the tables'' and use the test martingale
as a protection against the kind of changes that the test martingale benefits from.
If and when those changes happen,
the product (\emph{protected probability process}) of the base probability process and the test martingale
outperforms the base probability process;
equivalently, the \emph{protected predictive system},
given by Algorithm~\ref{alg:CJ-predictor} applied to a base predictive system,
outperforms the base predictive system in terms of the log loss function.

We will use the notation
\[
  \Loss(p_1,\dots,p_n\mid y_1,\dots,y_n)
  :=
  \sum_{i=1}^n
  \lambda(y_i,p_i)
\]
for the cumulative log-loss of predictions $p_i\in[0,1]$ on labels $y_i\in\{0,1\}$,
where $\lambda$ is defined by \eqref{eq:log-loss}.
The protection provided by Algorithm~\ref{alg:CJ-predictor}, and similar procedures, has its price,
since the loss suffered by the protected predictive system can be greater
than the loss suffered by the base predictive system.
The \emph{price of protection} for Algorithm~\ref{alg:CJ-predictor} is defined to be
\begin{multline}\label{eq:PoP}
  \sup
  \bigl(
    \Loss(p'_1,\dots,p'_n\mid y_1,\dots,y_n)
    -
    \Loss(p_1,\dots,p_n\mid y_1,\dots,y_n)
  \bigr)\\
  =
  \sup
  \left(
    -\ln S_n
  \right),
\end{multline}
where the $\sup$ is over all $n$, all object spaces $\mathbf{X}$,
all base predictive systems (outputting predictions $p_1,p_2,\dots$),
and all sequences of observations (with $y_i$ as their labels).
Of course, the definition given by the left-hand side of~\eqref{eq:PoP}
is applicable to any system transforming predictions $p_1,p_2,\dots$ to predictions $p'_1,p'_2,\dots$;
for Algorithm~\ref{alg:CJ-predictor} we have an equivalent definition
given by the right-hand side of~\eqref{eq:PoP}, $S$ being the Composite Jumper martingale.
We are particularly interested in the case of a finite price of protection.

\section{A theoretical guarantee}
\label{sec:theory}

A simple performance guarantee for Algorithm~\ref{alg:CJ-predictor}
is given by the following result (proved in Appendix~\ref{app:proofs}).
\begin{theorem}\label{thm:main}
  The price of protection for Algorithm~\ref{alg:CJ-predictor} is $\log\frac{1}{\pi}$.
  Besides,
  for any $n$, any jumping rate $J$, any sequence of observations,
  and any sequence $f_{\theta_1},\dots,f_{\theta_n}$ of calibrating functions
  (from the family $(f_{\theta}\mid\theta\in\Theta)$),
  \begin{multline}\label{eq:main}
    \Loss(p'_1,\dots,p'_n\mid y_1,\dots,y_n)\\
    \le
    \Loss(f_{\theta_1}(p_1),\dots,f_{\theta_n}(p_n)\mid y_1,\dots,y_n)\\
    +
    \log\frac{1}{1-\pi}
    +
    \log\left|\mathbf{J}\right|
    +
    k \log(\left|\Theta\right|-1)\\
    +
    k \log\frac{1}{J'}
    +
    (n-k-1) \log\frac{1}{1-J'},
  \end{multline}
  where
  \begin{equation}\label{eq:J}
    J'
    :=
    \frac{\left|\Theta\right|-1}{\left|\Theta\right|}
    J
  \end{equation}
  and $k=k(\theta_1,\dots,\theta_n)$ is the number of switches,
  \[
    k
    :=
    \left|\left\{
      i\in\{1,\dots,n\}: \theta_i \ne \theta_{i-1}
    \right\}\right|,
  \]
  with $\theta_0:=\boldsymbol{0}$.
\end{theorem}

The price of protection $\log\frac{1}{\pi}$ in Theorem~\ref{eq:main} is small unless $\pi$ is very close to 0,
even for test sets of a moderate size.
For example, setting $\pi:=0.5$ appears a reasonable compromise between the two terms involving $\pi$ in the price of protection
and in \eqref{eq:main}, and we will use it in our experiments in the next section.

The value $J'$ introduced in \eqref{eq:J} is an alternative parameterization of the jumping rate
(and it is used in, e.g., \cite{Vovk:1999derandomizing} as the main one);
it is usually close to (or at least has the same order of magnitude as) $J$.
We may call $J'$ the effective jumping rate:
it is the probability that the state actually changes at a given step.

The \emph{regret term} in the last two lines of \eqref{eq:main} can be interpreted as follows.
First, it gives the degree to which we are competitive with an ``oracular'' predictive system
that knows in advance which calibration function should be used at each step.
We have already discussed the addend $\log\frac{1}{1-\pi}$,
and $\log\left|\mathbf{J}\right|$ is the price,
typically very moderate ($\log 3$ in our experiments),
that we pay for using several jumping rates.
The following two addends in the regret term give us the price,
in terms of the log loss, for each switch.
Namely, each switch costs us $\log(\left|\Theta\right|-1)$ (which is $\log 8$ or $\log 20$ in our experiments)
plus an amount that depends on the switching rate, namely $\log\frac{1}{J'}$.
The last term in \eqref{eq:main} is close, assuming $k\ll n$ and $J'\ll1$, to $n J'$.
A reasonable choice of $J'$ is the inverse $1/N$ of the expected number $N$ of observations in the test set
(but remember that Algorithm~\ref{alg:CJ-predictor} covers a range of $J$,
which is motivated by $N$ typically not being known in advance).
With this choice the price $\log\frac{1}{J'}$ to pay per switch becomes $\log N$.

\begin{remark}
  The Markov chains usually used in tracking the best expert
  (see, e.g., \cite[Section 3.3]{Vovk:1999derandomizing})
  choose the first state randomly with equal probabilities,
  whereas our Simple Jumper Markov chain starts from the neutral state $\boldsymbol{0}\in\Theta$.
  This expresses the ``presumption of innocence'' towards the base predictive system:
  we do not calibrate it unless calibration is needed
  (and we definitely do not calibrate it at the very beginning of the test set).
  The two approaches lead to extremely similar results in our experiments.
\end{remark}

\section{Experimental results}
\label{sec:experimental}

There are not many datasets suitable for our experiments.
First, they should be ordered chronologically (to fit the scenario of Section~\ref{sec:introduction}),
or at least contain timestamps for all observations.
And second, they should not be of the time-series type,
so that applying typical machine-learning prediction algorithms
(such as those implemented in \texttt{scikit-learn}) makes sense.

Our main dataset will be \texttt{Bank Marketing}
(the only dataset in the top twelve most popular datasets at the UC Irvine Machine Learning Repository
that satisfies our requirements;
we will use, however, the full version of this dataset as given at the \texttt{openml.org} repository,
which is easier in \texttt{scikit-learn}).
The dataset consists of 45,211 observations representing telemarketing calls for selling long-term deposits
offered by a Portuguese retail bank, with data collected from 2008 to 2013 \cite{Moro/etal:2014}.
The labels are 1 or 2, which we encode as 0 and 1, respectively
(and we will never mention the original labels again in this paper).
Label 1 indicates a successful sale, and such observations comprise only 12\% of all labels.

The observations are listed in chronological order.
We take the first 10,000 observations as the training set,
normalize the attributes of the objects using \texttt{StandardScaler} in \texttt{scikit-learn}
(although normalization barely affects our results),
and train Random Forest with default parameters and random seed 2021 on it.
(In our figures in this section we use Random Forest as the base prediction algorithm,
since it consistently produces good results in our experiments.)
Random Forest often outputs probabilities of success that are equal to 0 or 1,
and when such a prediction turns out to be wrong (which happens repeatedly),
the log-loss is infinite.
It is natural, therefore, to truncate a probability $p\in[0,1]$ of 1 to the interval $[\epsilon,1-\epsilon]$
replacing $p$ by
\begin{equation}\label{eq:binary-truncation}
  p^*
  :=
  \begin{cases}
    \epsilon & \text{if $p\le\epsilon$}\\
    p & \text{if $p\in(\epsilon,1-\epsilon)$}\\
    1-\epsilon & \text{if $p\ge1-\epsilon$},
  \end{cases}
\end{equation}
where we set $\epsilon:=0.01$
(in \texttt{scikit-learn}, $\epsilon=10^{-15}$,
but it appears excessive to us).
The resulting prediction model is our base predictive system.
After we find it, we never use the training set again,
and the numbering of observations starts from the first element of the test set
(i.e., the dataset in the chronological order without the training set).

\begin{figure}
  \begin{center}
    \includegraphics[width=0.48\textwidth]{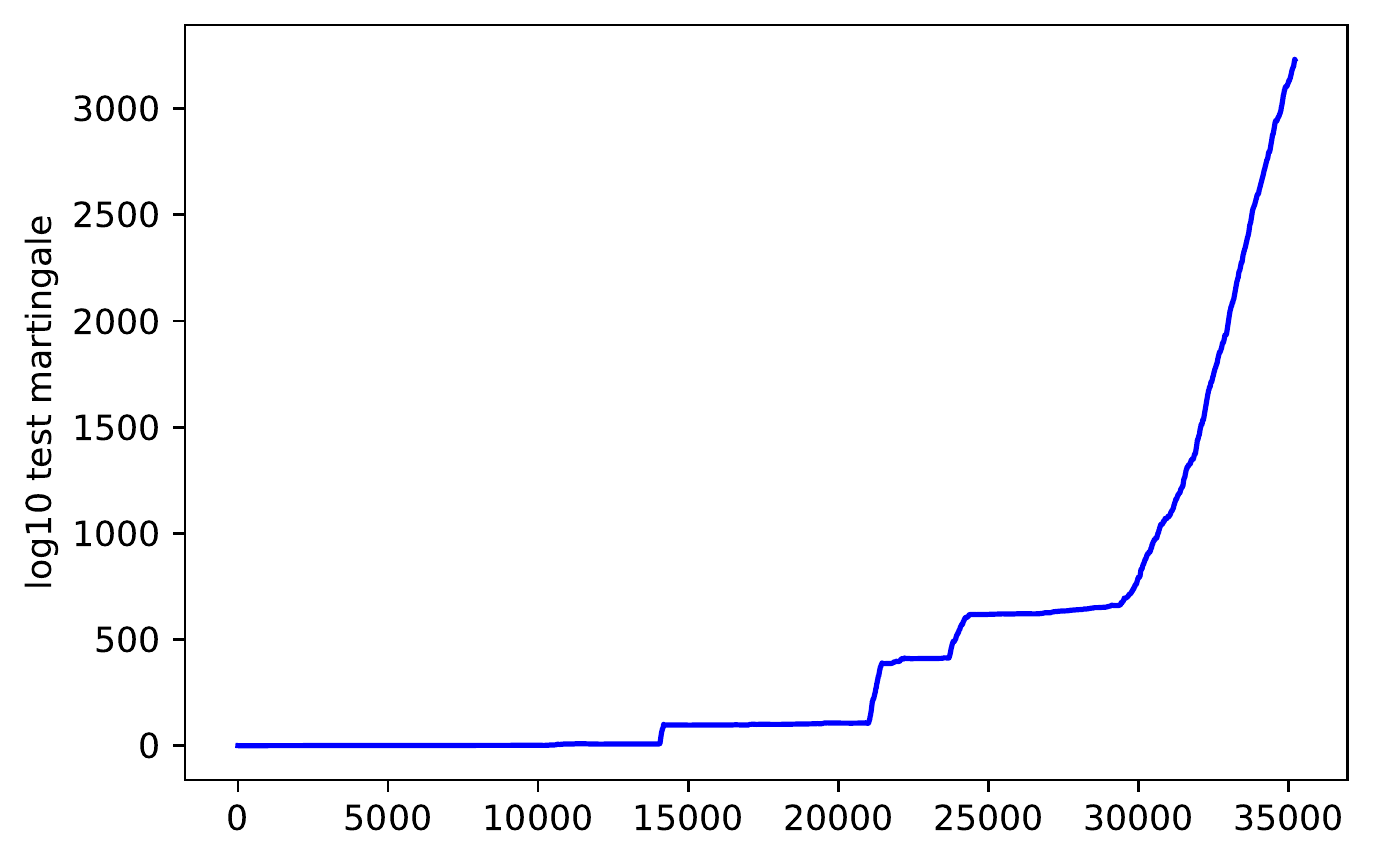}
    \includegraphics[width=0.48\textwidth]{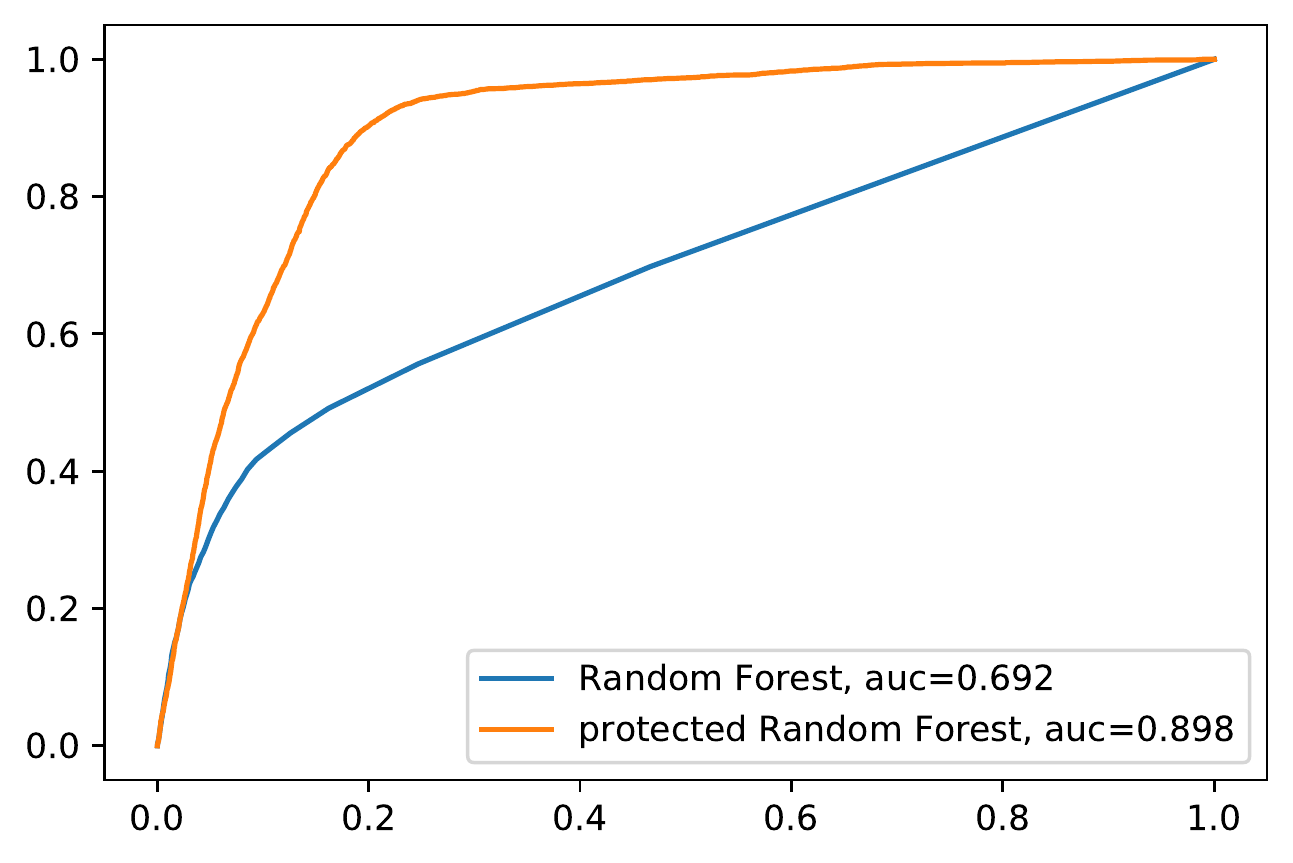}
  \end{center}
  \caption{Left panel: The Composite Jumper test martingale
    for the \texttt{Bank Marketing} dataset and Random Forest.
    Right panel: The ROC curve for the Composite Jumper protection.}
  \label{fig:BM_Cox}
\end{figure}

The left panel of Figure~\ref{fig:BM_Cox}
shows the trajectory of $\log_{10}S_n$, $n=1,\dots,35211$,
where $S_n$ is the value of the Composite Jumper test martingale over the test set
with the jumping rates $\mathbf{J}:=\{10^{-2},10^{-3},10^{-4}\}$
and the family \eqref{eq:Cox} with $\alpha\in\{-1,0,1\}$ and $\beta\in\{0.5,1,2\}$.
With this choice of the ranges of $\alpha$ and $\beta$,
which we always use in the binary case,
there are $\left|\Theta\right|=9$ of parameter vectors $\theta:=(\alpha,\beta)$.
The final value of the test martingale in the left panel of Figure~\ref{fig:BM_Cox}
is approximately $10^{3231.7}$.

The right panel of Figure~\ref{fig:BM_Cox} gives the ROC curve for Random Forest
and Random Forest protected by Algorithm~\ref{alg:CJ-predictor}.
We can see that the improvement is substantial.
In terms of the log-loss function and decimal logarithms,
the loss goes down from $7185.1$ to $3953.4$
(the difference between these two numbers being, predictably,
the exponent $3231.7$ in the final value of the test martingale in the left panel).

\begin{table}
  \caption{The AUC for the \texttt{Bank Marketing} dataset
    and key \texttt{scikit-learn} functions (with default parameters)
    together with their protected versions.}
  \label{tab:BM}
  \begin{center}
    \begin{tabular}{l|cc}
      \textbf{Prediction algorithm} & \textbf{base} & \textbf{protected} \\
      \hline
      Random Forest & 0.692 & 0.898 \\
      Gradient Boosting & 0.734 & 0.901 \\
      Decision Trees & 0.564 & 0.814 \\
      Neural Network & 0.665 & 0.879 \\
      SVM & 0.686 & 0.844 \\
      Naive Bayes & 0.646 & 0.807 \\
      Logistic Regression & 0.609 & 0.838
    \end{tabular}
  \end{center}
\end{table}

Table~\ref{tab:BM} gives the AUC (area under curve) for Algorithm~\ref{alg:CJ-predictor}
and several base prediction algorithms implemented in \texttt{scikit-learn}.
Since the dataset is imbalanced, AUC is a more suitable measure of quality than error rate.

\begin{figure}
  \begin{center}
    \includegraphics[width=0.6\textwidth]{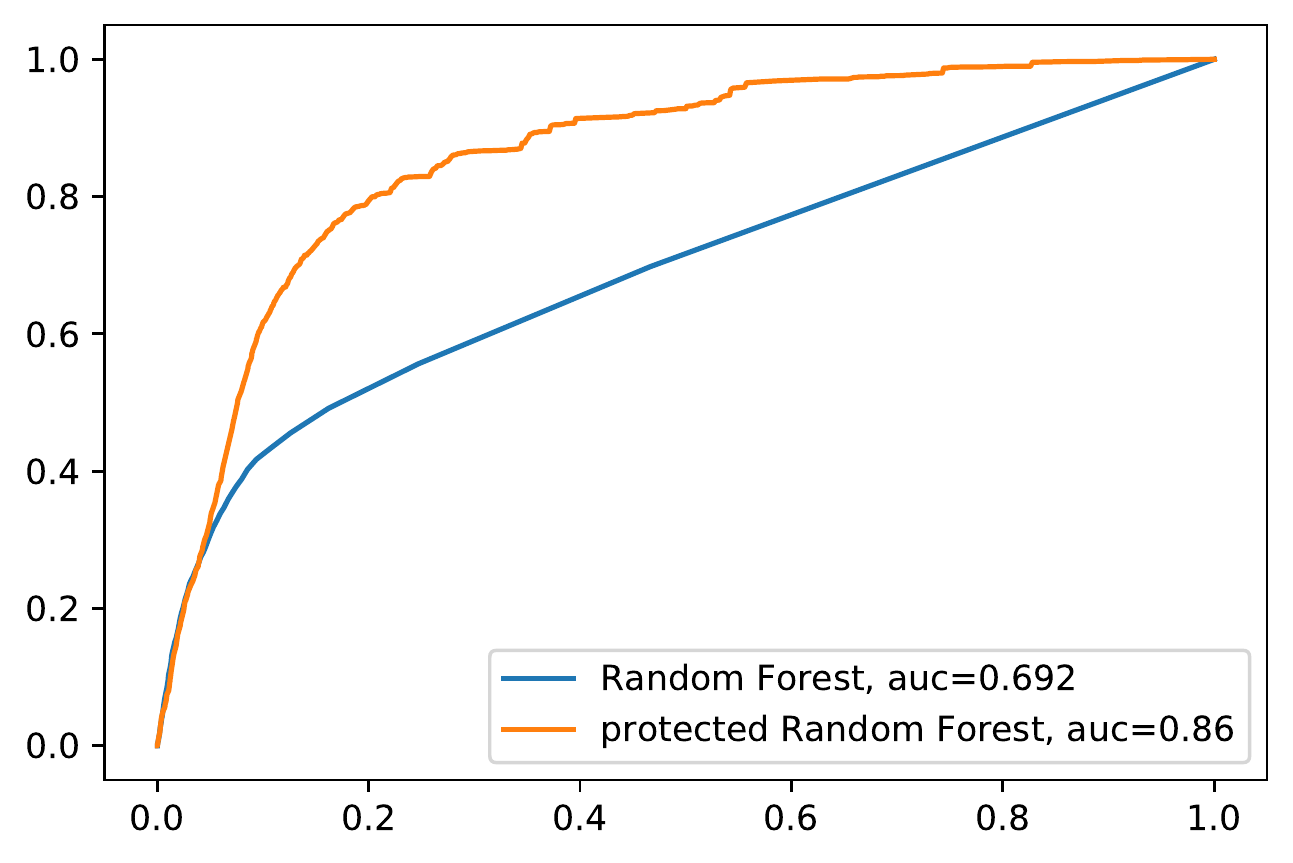}
  \end{center}
  \caption{The ROC curve for the \texttt{Bank Marketing} dataset
    and Random Forest with the Composite Jumper protection
    and feedback provided for every 100th test observation.}
  \label{fig:BM_Cox_ROC_lim}
\end{figure}

In our experiments so far we have assumed that every test observation
is used for recalibrating the prediction model.
From the practical point of view this may be unrealistic,
and in reality we can only hope to get feedback on a fraction of the test observations.
Figure~\ref{fig:BM_Cox_ROC_lim} is the counterpart of the right panel of Figure~\ref{fig:BM_Cox}
for the case where the vast majority of observations are predicted by the prediction model
without getting any feedback,
and the weight updates in Algorithm~\ref{alg:CJ-predictor} are run only on every 100th test observation
(so that the same weights $P$ and $A^J_{\theta}$ are used
for the observations $100k+1$, $100k+2$,\dots, $100k+100$ for $k=0,1,\dots$).
Now the improvement is more modest.

\begin{figure}
  \begin{center}
    \includegraphics[width=0.48\textwidth]{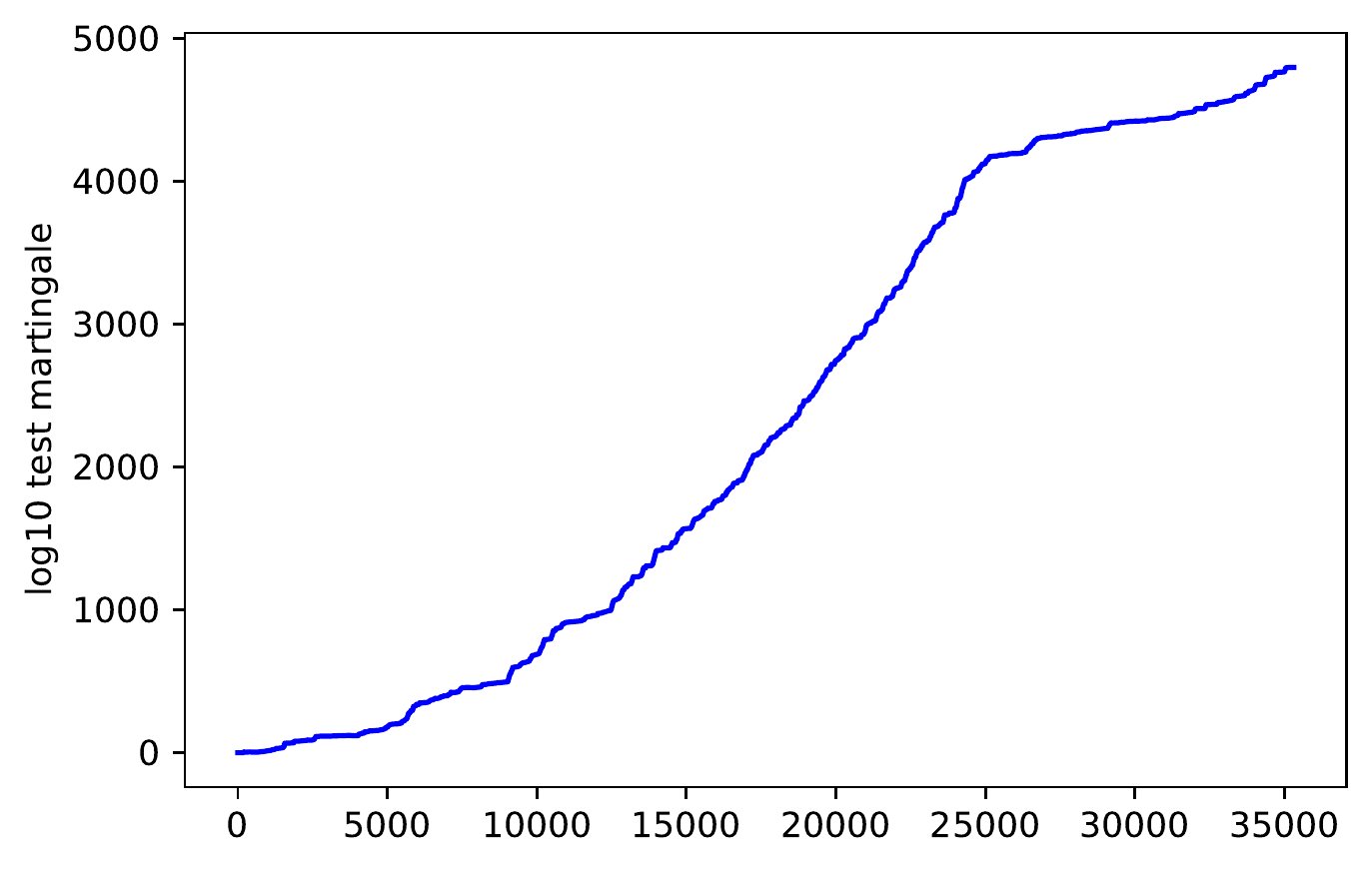}
    \includegraphics[width=0.48\textwidth]{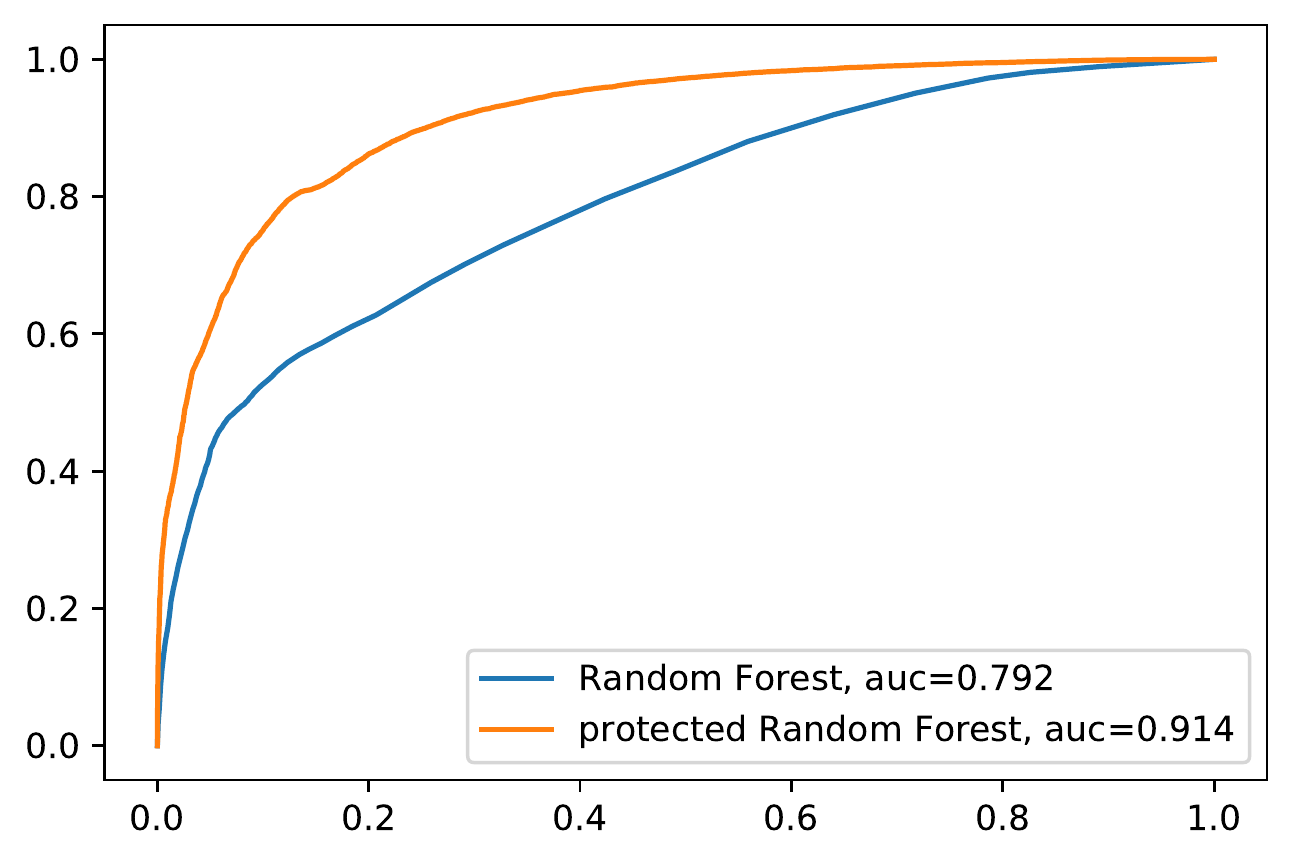}
  \end{center}
  \caption{The analogue of Figure~\ref{fig:BM_Cox} for the \texttt{electricity} dataset.}
  \label{fig:el_Cox}
\end{figure}

Another similar dataset is \texttt{electricity} \cite{Harries:1999,Gama/etal:2004}, available from \texttt{openml.org}.
Its 45,312 observations contain binary labels
(whether the electricity price in New South Wales goes up or down,
encoded as 1 and 0, respectively)
together with relevant attributes,
collected from 7 May 1996 to 5 December 1998.
The observations are listed in the chronological order, as for \texttt{Bank Marketing},
and we use the same scheme, allocating the first 10,000 observations to the training set
and normalizing the attributes with \texttt{StandardScaler};
the figures use the same base prediction algorithm (Random Forest).
Figure~\ref{fig:el_Cox} shows results for this dataset;
protection still greatly improves the performance of the base predictions.

\begin{table}
  \caption{Numbers of errors for the \texttt{electricity} dataset
    and some \texttt{scikit-learn} functions (with default parameters)
    together with their protected versions.}
  \label{tab:electricity}
  \begin{center}
    \begin{tabular}{l|cc}
      \textbf{Prediction algorithm} & \textbf{base} & \textbf{protected} \\
      \hline
      Random Forest & 9184 & 5846 \\
      Gradient Boosting & 9165 & 6009 \\
      Decision Trees & 9372 & 6806 \\
      Neural Network & 14,358 & 7469 \\
      SVM & 14,433 & 9532
    \end{tabular}
  \end{center}
\end{table}

Table~\ref{tab:electricity} gives results for Algorithm~\ref{alg:CJ-predictor}
and several base prediction algorithms
(\texttt{scikit-learn}'s Naive Bayes and Logistic Regression fail on this dataset
producing inconsistent results).
Now the dataset is balanced, and so we report the numbers of errors
(i.e., the cases of the predictions being different from the true labels).

\subsection*{Multiclass case}

Here we work with the \texttt{UJIIndoorLoc} dataset \cite{Torres/etal:2014},
available from the UC Irvine Machine Learning Repository.
The attributes are intensity levels of 520 wireless access points (WAPs) in three buildings
of the Jaume I University in Castell\'o de la Plana, Valencia, Spain.
The task we are interested in is to identify the building given the WAP intensity levels.
The dataset consists of two parts, a sizable original training set and a much smaller original validation set
(the latter collected 4 months after the former).
Since the attributes that we use are given on the same scale, there is no need to normalize them.

We consider two scenarios:
\begin{itemize}
\item
  In Scenario 1, we ignore the original validation set and use only the original training set,
  which we order chronologically and then split into the training set consisting of the first 10,000 observations
  and the test set consisting of the remaining observations,
  as we did for \texttt{Bank Marketing} and \texttt{electricity}.
\item
  In Scenario 2, we use the original training set as our training set
  and the original validation set as our test set.
\end{itemize}

In this multiclass case
we truncate a probability measure $p=(p_y)$ slightly differently from the binary procedure~\eqref{eq:binary-truncation};
namely, we set
\begin{equation}\label{eq:multiclass-truncation}
  p^*_y
  :=
  \frac{\max(p_y,\epsilon)}{\sum_{y'\in\mathbf{Y}}\max(p_{y'},\epsilon)},
  \quad
  y\in\mathbf{Y}
\end{equation}
(where $\mathbf{Y}$ are the labels standing for the three buildings, $\left|\mathbf{Y}\right|=3$),
and we still set $\epsilon:=0.01$.

\begin{table}
  \caption{Numbers of errors in Scenario 1 for the \texttt{UJIIndoorLoc} dataset
    and key \texttt{scikit-learn} functions (with default parameters).}
  \label{tab:scenario-1}
  \begin{center}
    \begin{tabular}{l|cc}
      \textbf{Prediction algorithm} & \textbf{base} & \textbf{protected} \\
      \hline
      Random Forest & 556 & 208 \\
      Gradient Boosting & 1397 & 1041 \\
      Decision Trees & 1185 & 1185 \\
      Neural Network & 1289 & 1261 \\
      SVM & 354 & 199 \\
      Naive Bayes & 43 & 43 \\
      Logistic Regression & 290 & 244
    \end{tabular}
  \end{center}
\end{table}

First we report our results for the more difficult Scenario~1.
It is interesting that one of the buildings is not in the test set;
it seems that the buildings were explored systematically.
The results for Scenario 1 are given in Table~\ref{tab:scenario-1},
where $\beta$ range over $\{0.5,1,2\}$ and $\alpha$ range over the 7 binary vectors of length 3
apart from $(1,1,1)$ (which is not included since the corresponding calibrating function
is the same as for $(0,0,0)$);
of course, the neutral calibrating function is the one with $\alpha=(0,0,0)$ and $\beta=1$.

Protection always improves results, sometimes significantly, apart from Decision Trees and Naive Bayes,
for which the number of errors stays the same, 1185 and 43, respectively.
In the case of Decision Trees, the vast majority of the base predictions are categorical,
concentrating on one label, which makes their calibration problematic.
(Namely, 9251 out of 9937 predictions are categorical, assigning probability 1 to one of the labels;
in particular, all wrong predictions are categorical.
The least categorical of the remaining predictions assigns probability 0.974 to one of the labels.)
In the case of Naive Bayes, the quality of the probabilistic predictions still improves greatly:
the log10 loss of the base predictive system is $\infty$
because the loss is $\infty$ for 8 observations;
if those 8 observations are ignored, the log10 loss is $2304.76$,
whereas the log10 loss of the protected predictive system is $84.41$ (without any infinities).

\begin{table}
  \caption{Numbers of errors in Scenario 2 for the \texttt{UJIIndoorLoc} dataset
    and \texttt{scikit-learn} functions with default parameters.}
  \label{tab:scenario-2}
  \begin{center}
    \begin{tabular}{l|cc}
      \textbf{Prediction algorithm} & \textbf{base} & \textbf{protected} \\
      \hline
      Random Forest & 2 & 1 \\
      Gradient Boosting & 2 & 2 \\
      Decision Trees & 21 & 10 \\
      Neural Network & 1 & 1 \\
      SVM & 4 & 4 \\
      Naive Bayes & 9 & 9 \\
      Logistic Regression & 0 & 0
    \end{tabular}
  \end{center}
\end{table}

Scenario 2 is extremely easy (since all three buildings have been explored completely);
e.g., Logistic Regression as implemented in \texttt{scikit-learn} does not make any errors.
For the results, see Table~\ref{tab:scenario-2}.
It is reassuring that the number of errors never goes up as result of protection.

\section{Conclusion}
\label{sec:conclusion}

The methods of adaptive calibration that we propose in this paper need to be validated
on other datasets and perhaps for other calibrating functions.
Notice that calibrating functions may depend not only on the current predicted probability $p$
but also on the current object $x$.
(So that ``calibration'' may be understood in a very wide sense, as in \cite{Dawid:1985},
and include elements of ``resolution'' \cite{Dawid:ESS2006PF}.)

In this paper we only use the log loss function.
Arguably it is the most fundamental one \cite{Vovk:2015Yuri},
but its disadvantage is that, for many base prediction algorithms,
we need truncation (see \eqref{eq:binary-truncation} and \eqref{eq:multiclass-truncation})
to prevent an infinite loss.
A popular alternative to the log loss function is the Brier loss function \cite{Brier:1950};
it is much more forgiving and does not require truncation.
It follows from the results of \cite{Vovk/Zhdanov:2009}
that Theorem~\ref{thm:main} will continue to hold
when the log loss function is replaced by the Brier loss function.
Empirical studies of the performance of our procedures in this case
are an interesting direction of further research.

A useful feature of our procedure of protection is that it is cheap,
which is achieved by mixing Simple Jumper martingales with constant 1: see \eqref{eq:average}.
The role of mixing with 1 is to insure against a catastrophic loss of evidence
against the null hypothesis (given by the base predictive system)
found by those test martingales.
There are much more sophisticated ways of insuring against loss of evidence \cite[Chapter 11]{Shafer/Vovk:2019},
and they will provide further protection.

\subsection*{Acknowledgments}

We are grateful to Glenn Shafer for his advice.
Phelype Oleinik has helped us with \TeX.
This research has been partially supported by Amazon and Stena Line.
In our computational experiments we have used \texttt{scikit-learn} \cite{scikit-learn:2011}.

\appendix

\section{Some proofs}
\label{app:proofs}

The second statement of Theorem~\ref{thm:main} can be deduced
from the fact that the probability process of the Bayes mixture
is equal to the average of the elementary predictors' probability processes
with respect to the prior probability measure on the elementary predictors.
This can be checked directly and is a special case of a known result for the AA
specialized to the log loss function.
In the context of this special case of the AA,
a probability process is a process of the form $\exp(-L)$, where $L$ is a loss process.
The following lemma is the main property of the AA in this context.

\begin{lemma}\label{lem:main}
  The probability process of the AA with the log loss function
  is the average of the elementary predictors' probability processes
  with respect to the prior probability measure.
\end{lemma}

\begin{proof}
  For a proof, see \cite[Lemma 1]{Vovk:2001competitive}.
\end{proof}

\subsection*{Proof of Theorem~\ref{thm:main}}

Let us first prove that the price of protection is $\log\frac{1}{\pi}$.
It is obvious that it does not exceed $\log\frac{1}{\pi}$,
and it suffices to prove that, for any $J\in\mathbf{J}$,
the likelihood ratio of the Simple Jumper predictor with jumping rate $J$ to the base predictive system
tends to 0 for some sequence of observations.
We will prove more: namely,
for a suitable choice of the base predictive system,
the likelihood ratio of the Simple Jumper predictor with jumping rate $J$ to the base predictive system
tends to 0 a.s.\ under the probability distribution of the base predictive system.

By the ergodic theorem for Markov chains (see, e.g., \cite[Theorem 1.10.2]{Norris:1997})
the Simple Jumper Markov chain will spend, asymptotically and almost surely,
the fraction $1/\left|\Theta\right|$ of its time in each state $\theta\in\Theta$.
Choose a $p$ such that $f_{\theta}(p)\ne p$ for some $\theta\in\Theta$.
Let the base predictive system always output $p$.
By Kabanov et al.'s criterion (see, e.g., \cite[Theorem 4]{Shiryaev:2019})
of mutual singularity of probability measures in ``predictable terms'',
the Simple Jumper predictive system with jumping rate $J$ and the base predictive system
will be mutually singular.
This implies (by \cite[Theorem 2]{Shiryaev:2019})
that, indeed, the likelihood ratio of the Simple Jumper to the base predictive system tends to 0 a.s.

To complete the proof of Theorem~\ref{thm:main},
notice that the probability that the Simple Jumper Markov chain with a given jumping rate $J\in\mathbf{J}$
produces $\theta_1,\dots,\theta_n$ with $k$ switches as its first $n$ states
is
\[
  \left(
    \frac{J'}{\left|\Theta\right|-1}
  \right)^k
  (1-J')^{n-k-1}.
\]
Therefore, the protected probability process is at least
\[
  \frac{1-\pi}{\left|\mathbf{J}\right|}
  \left(
    \frac{J'}{\left|\Theta\right|-1}
  \right)^k
  (1-J')^{n-k-1}
\]
times the probability process for any elementary predictor with jumping rate $J$
and the sequence of states beginning with $\theta_1,\dots,\theta_n$.
It remains to convert the probability processes into log loss processes
by applying the operation $-\log$.

\section{Further experimental results}
\label{app:further}

\begin{figure}
  \begin{center}
    \includegraphics[width=0.48\textwidth]{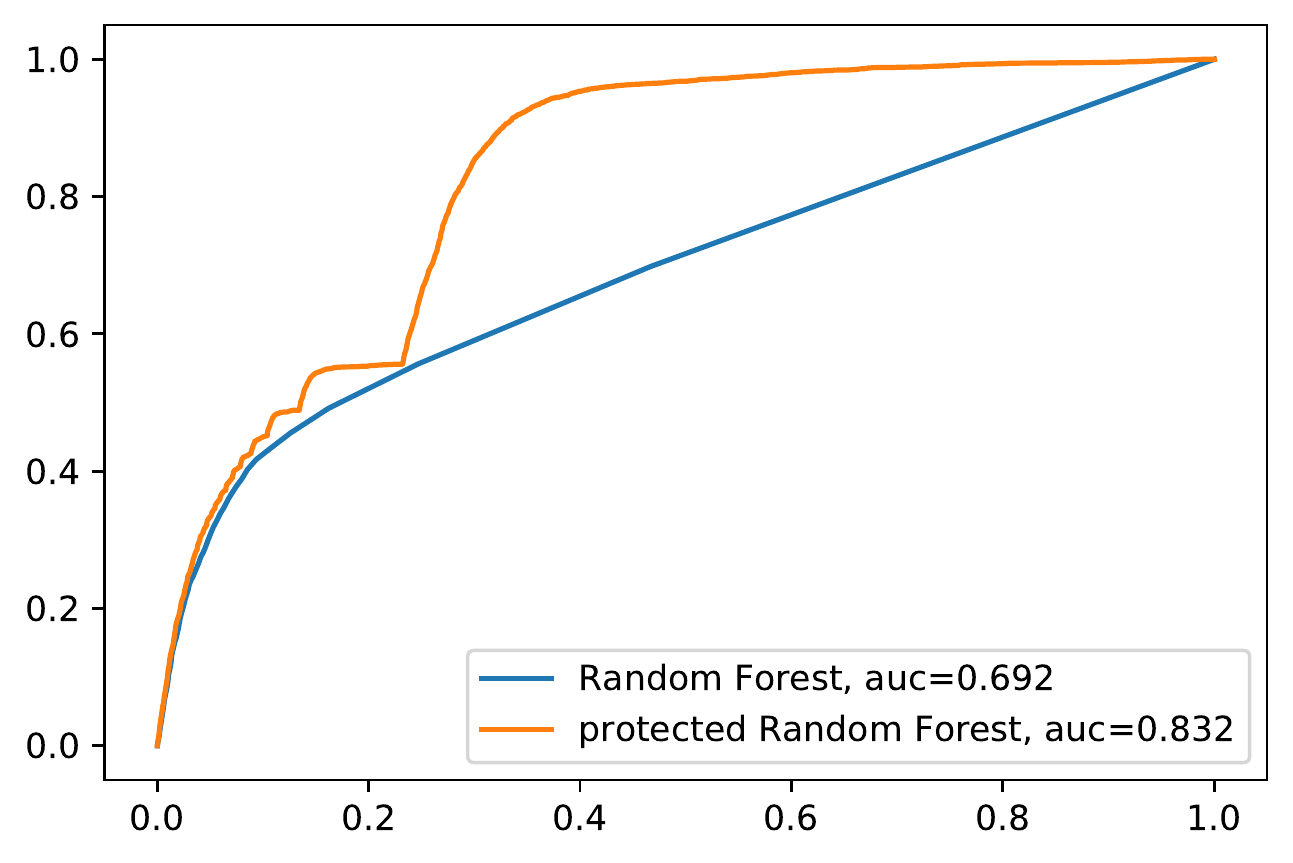}
    \includegraphics[width=0.48\textwidth]{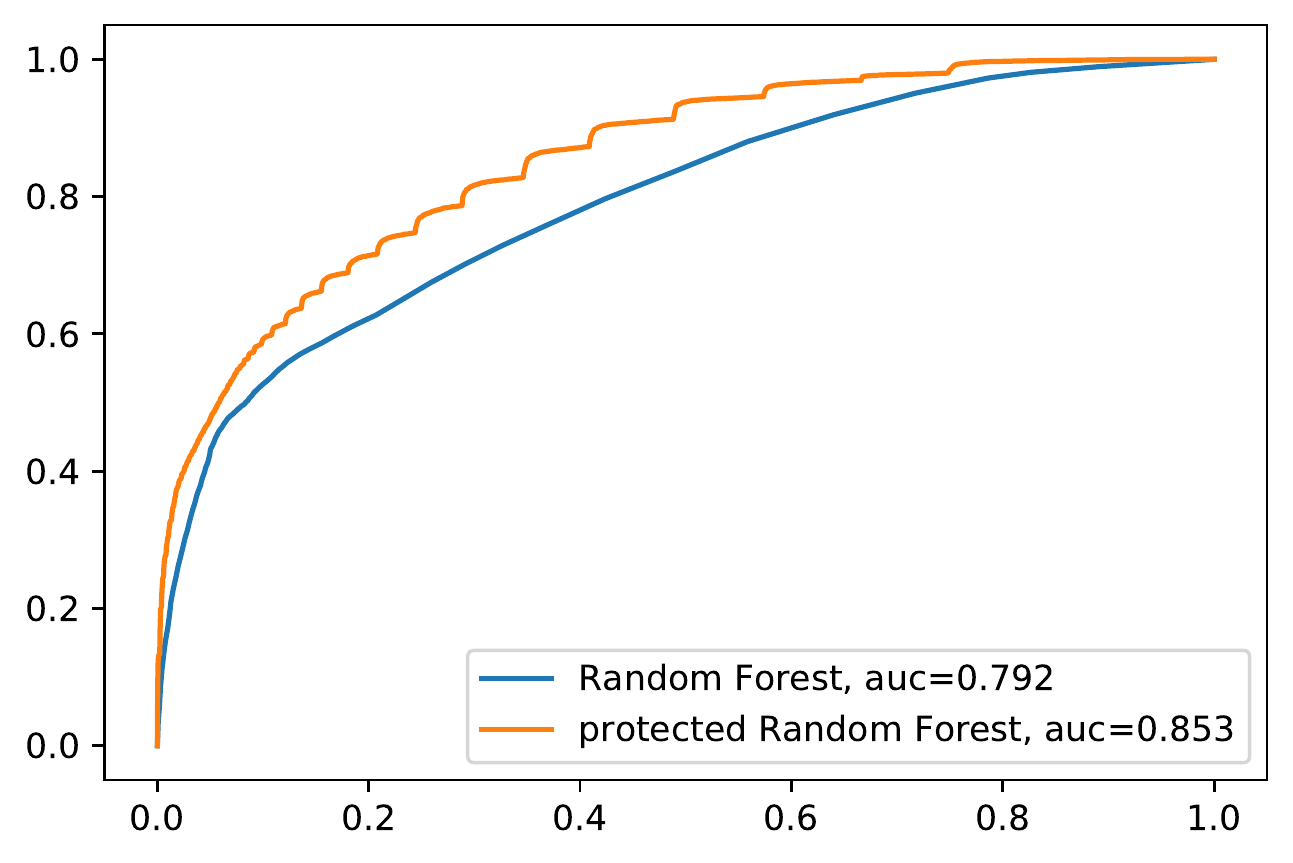}
  \end{center}
  \caption{The ROC curves for the \texttt{Bank Marketing} (left panel) and \texttt{electricity} (right panel) dataset
    and Random Forest with the Composite Jumper protection based on quadratic calibration.}
  \label{fig:quad_ROC}
\end{figure}

We start this appendix from reporting results
for the family \eqref{eq:quadratic} of quadratic calibrating functions
with $\theta$ restricted to a finite set $\Theta$.
As always, we choose a minimal $\Theta$, namely $\Theta:=\{-1,0,1\}$.
Figure~\ref{fig:quad_ROC} is the counterpart
of the right panels of Figures~\ref{fig:BM_Cox} and~\ref{fig:el_Cox} for this family.
In the rest of the paper we only consider Cox's calibrating functions.

\begin{figure}
  \begin{center}
    \includegraphics[width=0.48\textwidth]{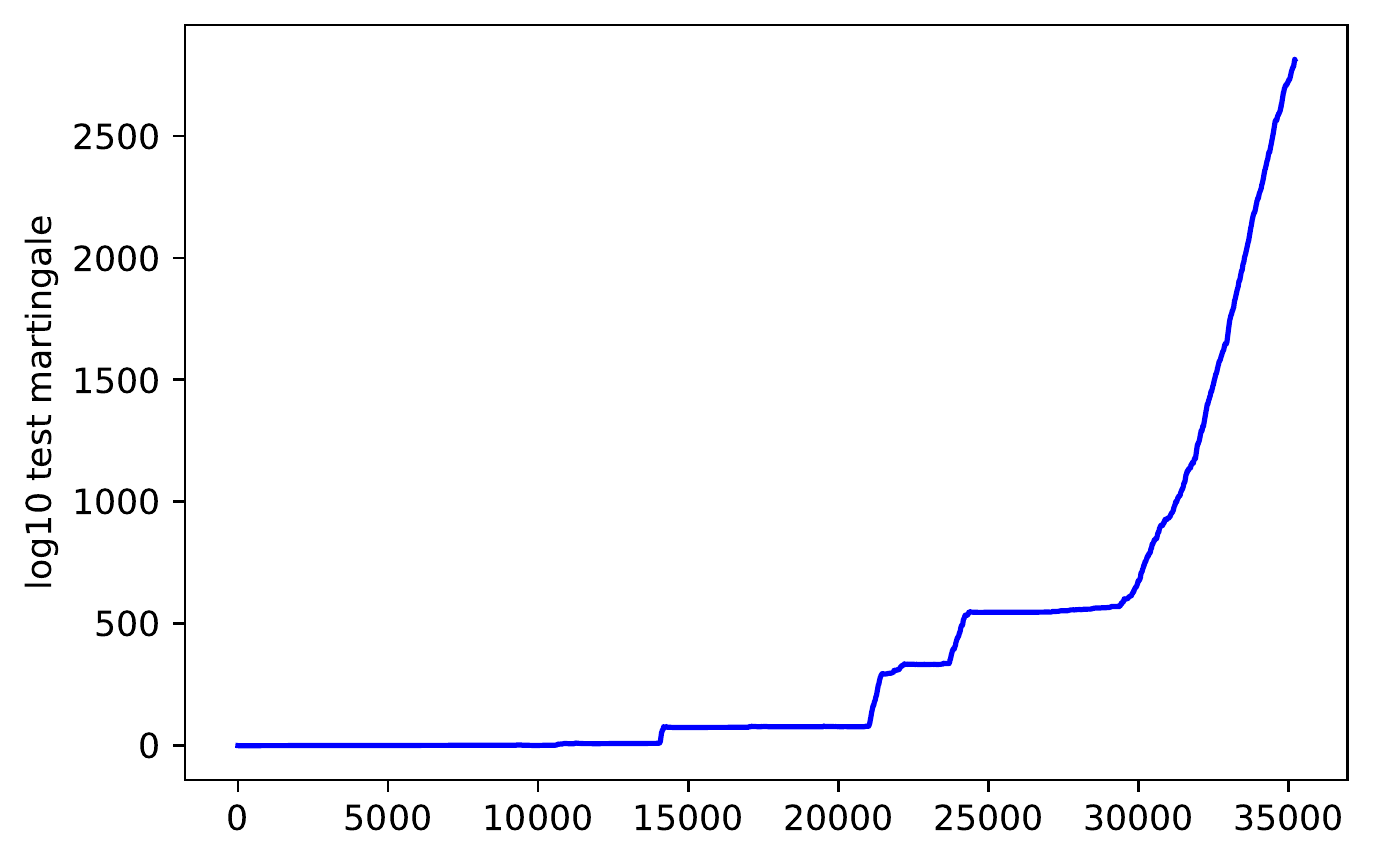}
    \includegraphics[width=0.48\textwidth]{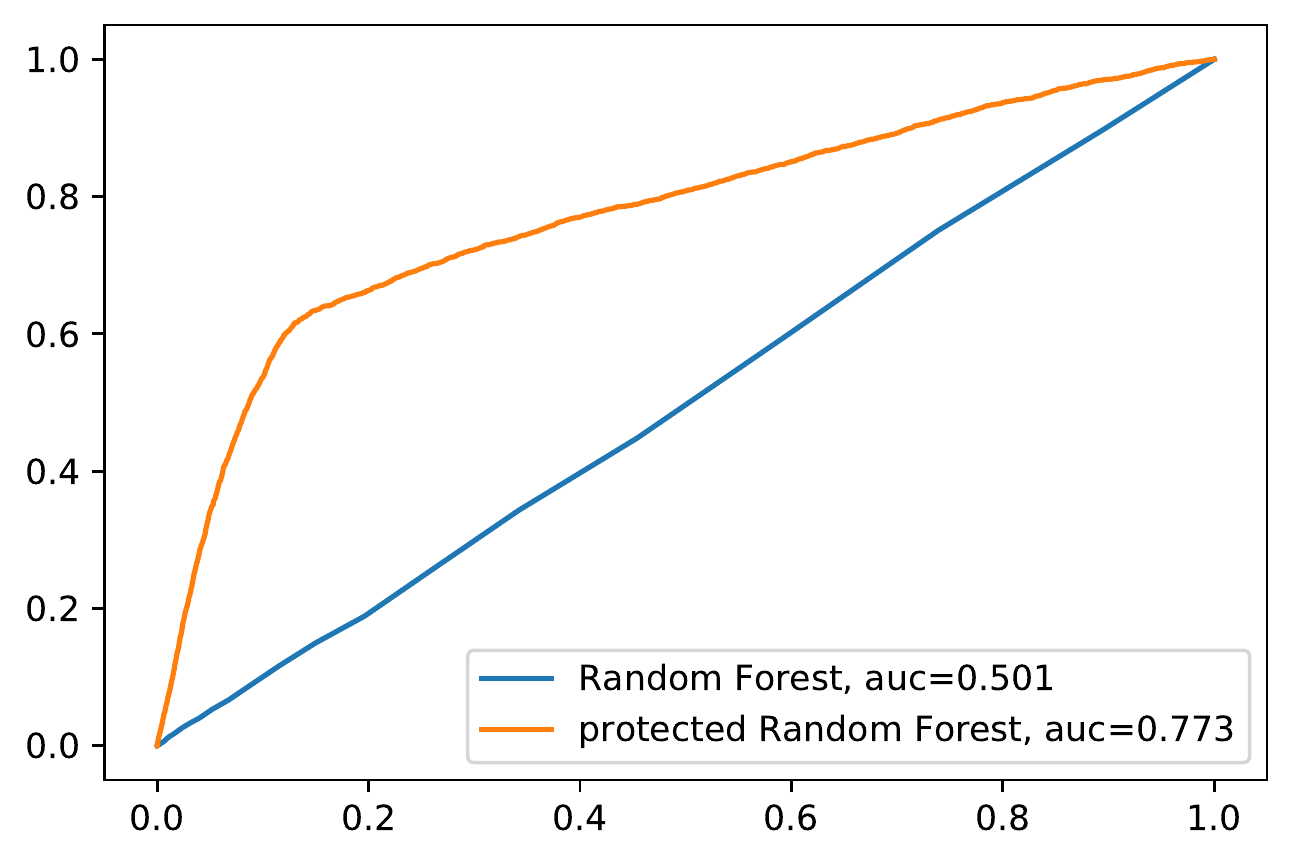}
  \end{center}
  \caption{The analogue of Figure~\ref{fig:BM_Cox} for the \texttt{Bank Marketing} dataset
    with randomly permuted objects.}
  \label{fig:BM_Cox_perm}
\end{figure}

\begin{figure}
  \begin{center}
    \includegraphics[width=0.48\textwidth]{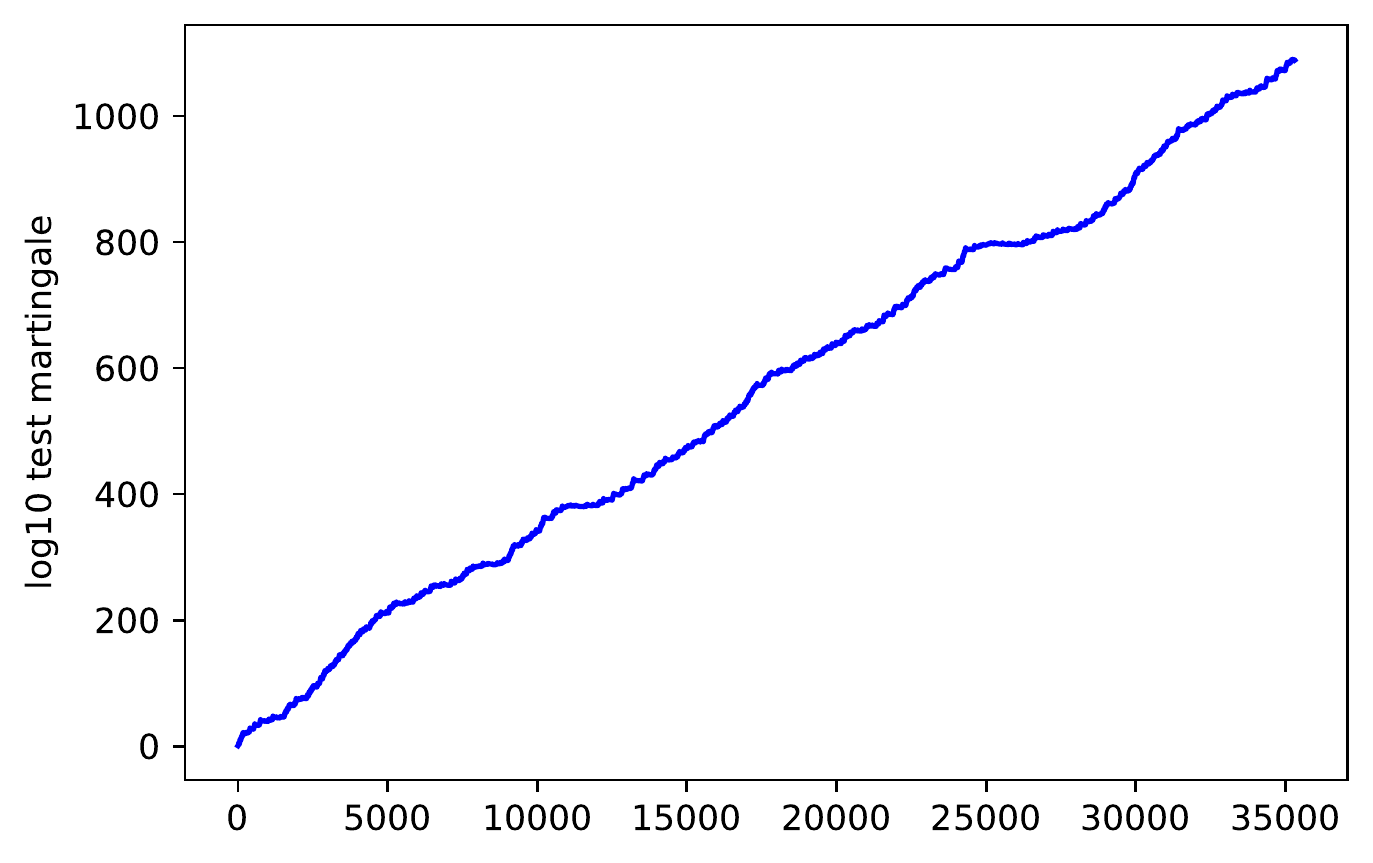}
    \includegraphics[width=0.48\textwidth]{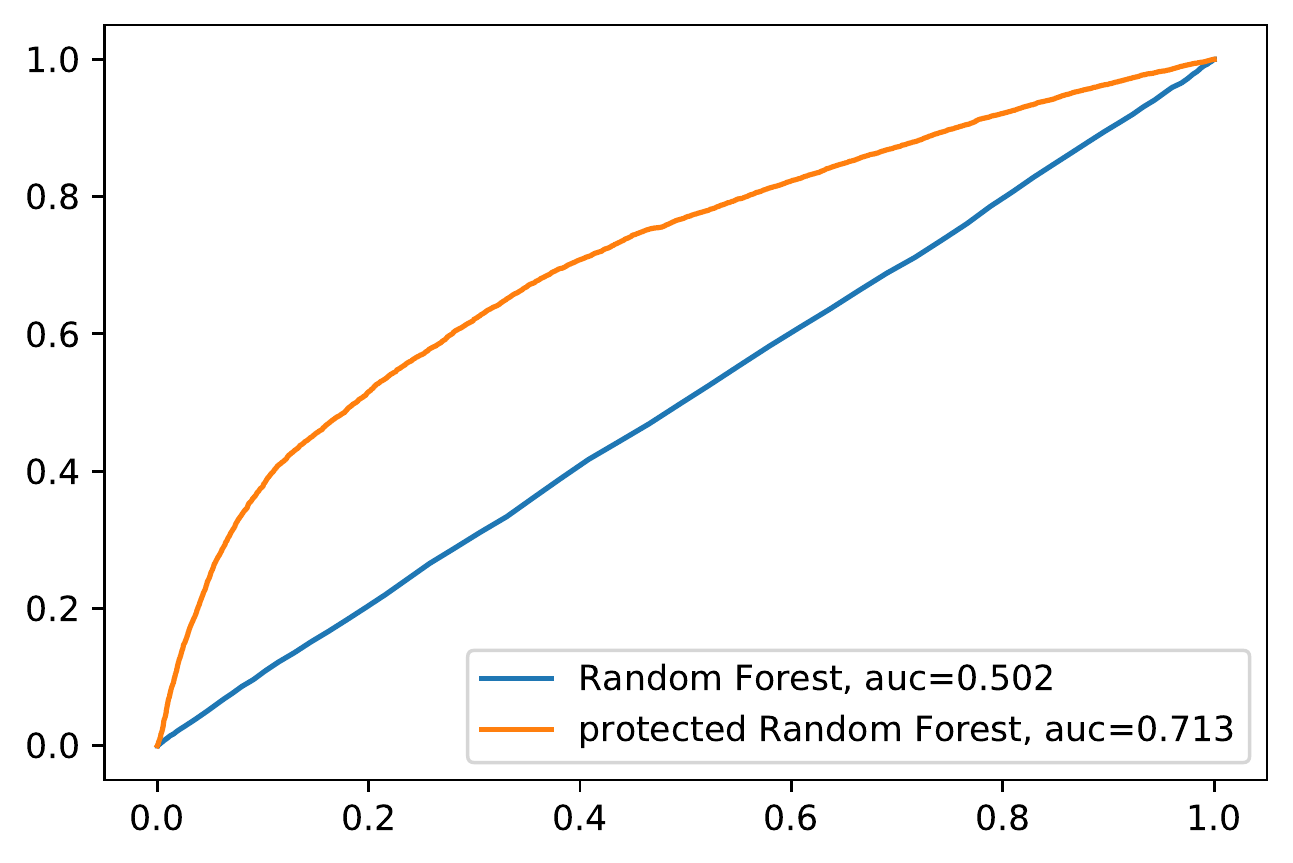}
  \end{center}
  \caption{The analogue of Figure~\ref{fig:el_Cox} for the \texttt{electricity} dataset
    with randomly permuted objects.}
  \label{fig:el_Cox_perm}
\end{figure}

Next we discuss some of the reasons for a good performance of our protection procedures
on the \texttt{Bank Marketing} and \texttt{electricity} datasets.
It is revealing that protection still gives good results for both datasets
when the objects are randomly shuffled (so that they become uninformative).
See Figures~\ref{fig:BM_Cox_perm} and~\ref{fig:el_Cox_perm}.
Therefore, already the order of the test labels is informative.
Of course, the ROC curves for the unprotected procedures
become trivial (close to the main diagonal) when the objects are shuffled.

\begin{figure}
  \begin{center}
    \includegraphics[width=0.48\textwidth]{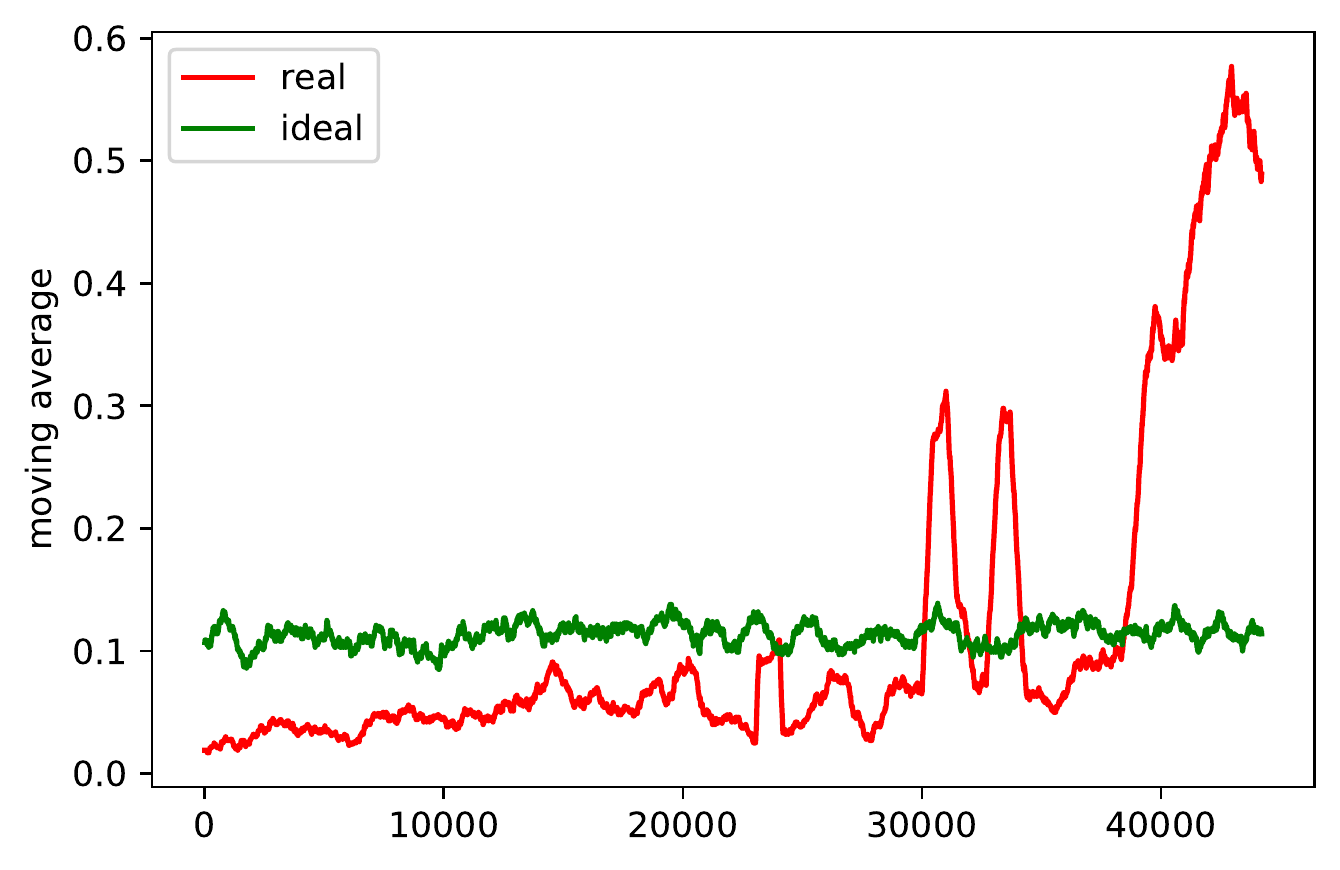}
    \includegraphics[width=0.48\textwidth]{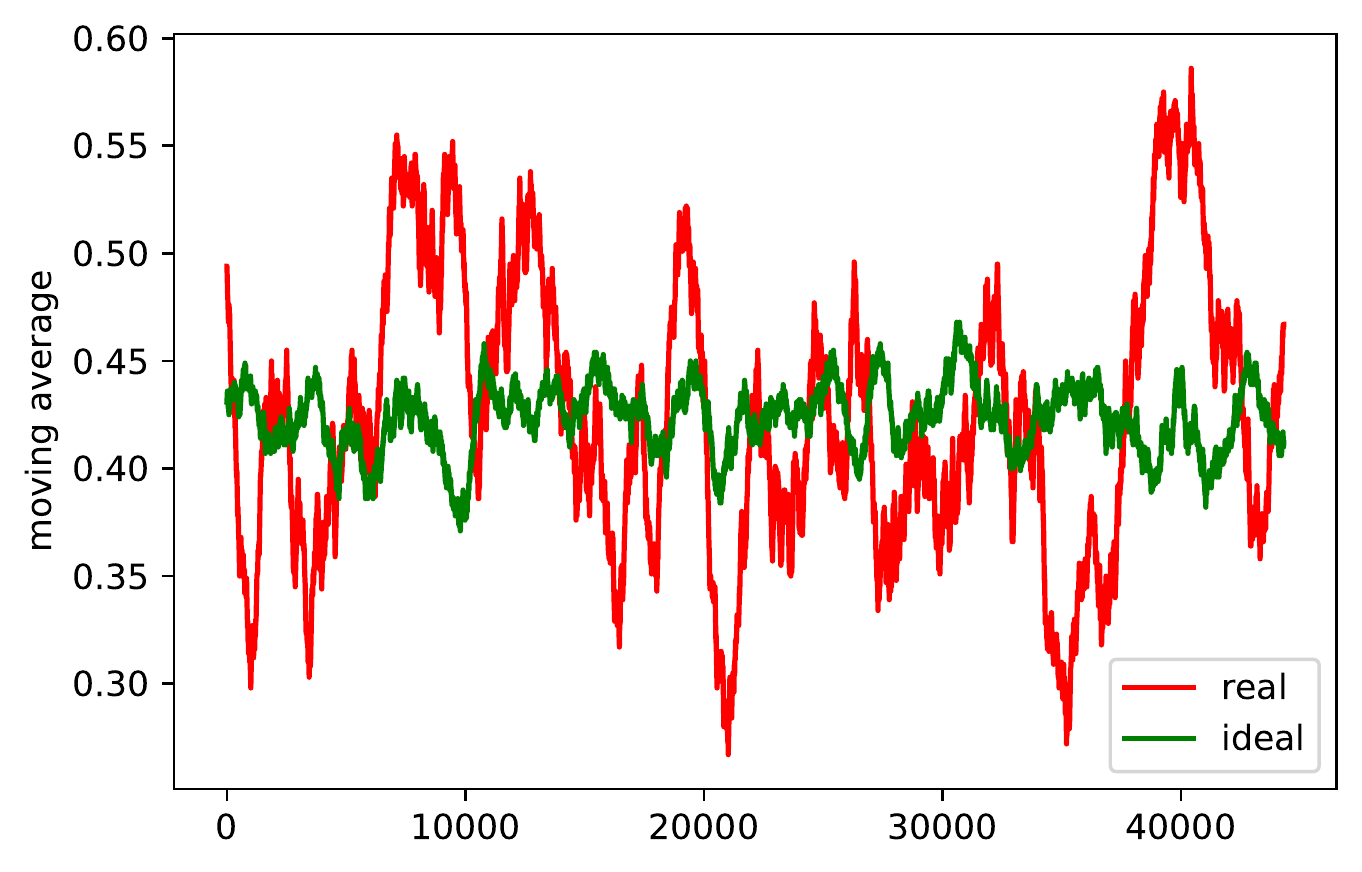}
  \end{center}
  \caption{The moving averages of the labels
    for the \texttt{Bank Marketing} (left panel) and \texttt{electricity} (right panel) datasets,
    as described in text.}
  \label{fig:MA}
\end{figure}

To understand the reasons for the ROC curves being non-trivial after protection in
Figures~\ref{fig:BM_Cox_perm} and~\ref{fig:el_Cox_perm},
we compute the moving averages of the labels for the two datasets (including both training and test sets).
Figure~\ref{fig:MA} shows in red the trajectories of the moving averages of the labels:
the value of each trajectory at time $n$ is the arithmetic mean of the 1000 consecutive labels starting from $y_n$
(namely, the arithmetic mean of $y_n,\dots,y_{n+999}$).
For comparison, the analogous moving average for a simulated IID binary sequence with the right percentage of 1s
(12\% for \texttt{Bank Marketing} and 42\% for \texttt{electricity})
is shown in green.

The behaviour of the moving average is particularly anomalous for \texttt{Bank Marketing}:
the proportion of successful calls increases drastically towards the end of the dataset,
which explains the quick growth of the Composite Jumper martingale in
Figures~\ref{fig:BM_Cox} and~\ref{fig:BM_Cox_perm} starting from approximately the 30,000th observation in the test set,
which corresponds to the 40,000th observation in the full dataset.
The percentage of successful calls is 3.5\% for the training set and 14.0\% for the test set.
The behaviour for \texttt{electricity} is less anomalous,
but there are still clear non-random patterns.

\subsection*{Dependence on the jumping rate}

\begin{figure}
  \begin{center}
    \includegraphics[width=0.48\textwidth]{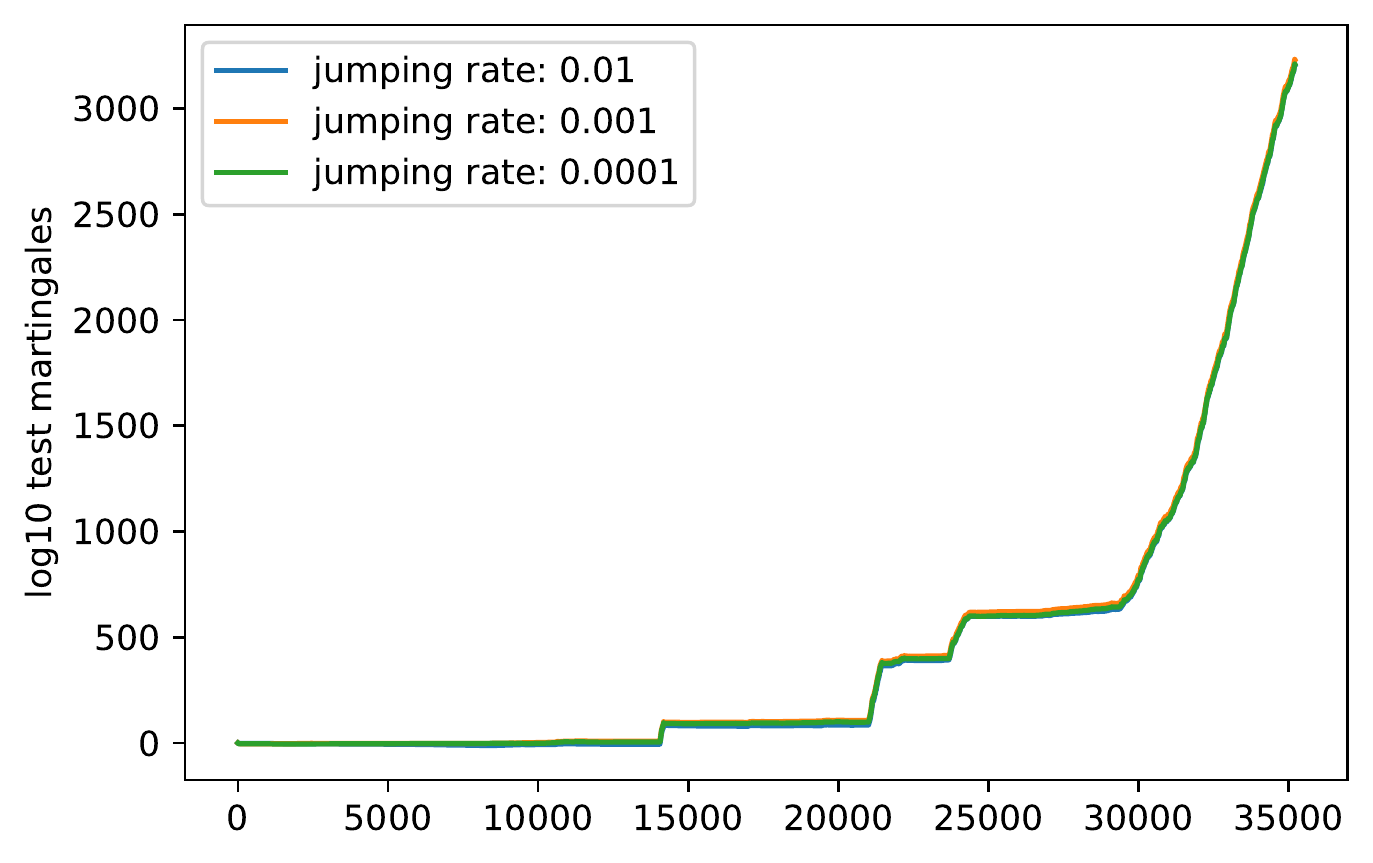}
    \includegraphics[width=0.48\textwidth]{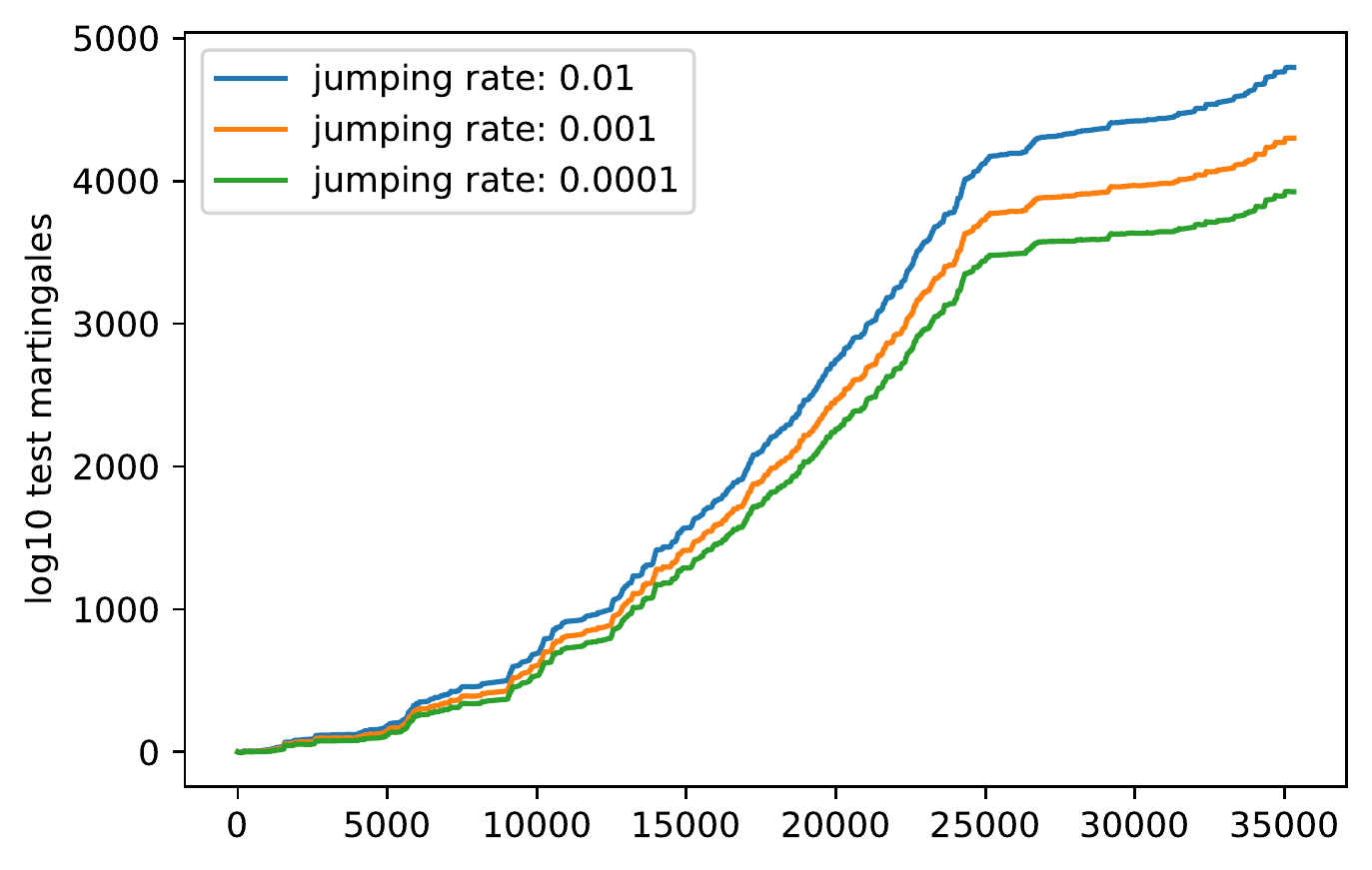}
  \end{center}
  \caption{The Simple Jumper martingales for various jumping rates
    for the \texttt{Bank Marketing} (left panel) and \texttt{electricity} (right panel) datasets,
    as described in text.}
  \label{fig:Cox_SJs}
\end{figure}

The left panel of Figure~\ref{fig:Cox_SJs} shows the dependence of the Simple Jumper martingale
(with the same parameters as in the left panel of Figure~\ref{fig:BM_Cox})
on the jumping rate for the \texttt{Bank Marketing} dataset;
the dependence is slight, at least on the log scale.
The right panel is its counterpart for the \texttt{electricity} dataset;
the dependence on the jumping rate is more noticeable.

\begin{figure}
  \begin{center}
    \includegraphics[width=0.48\textwidth]{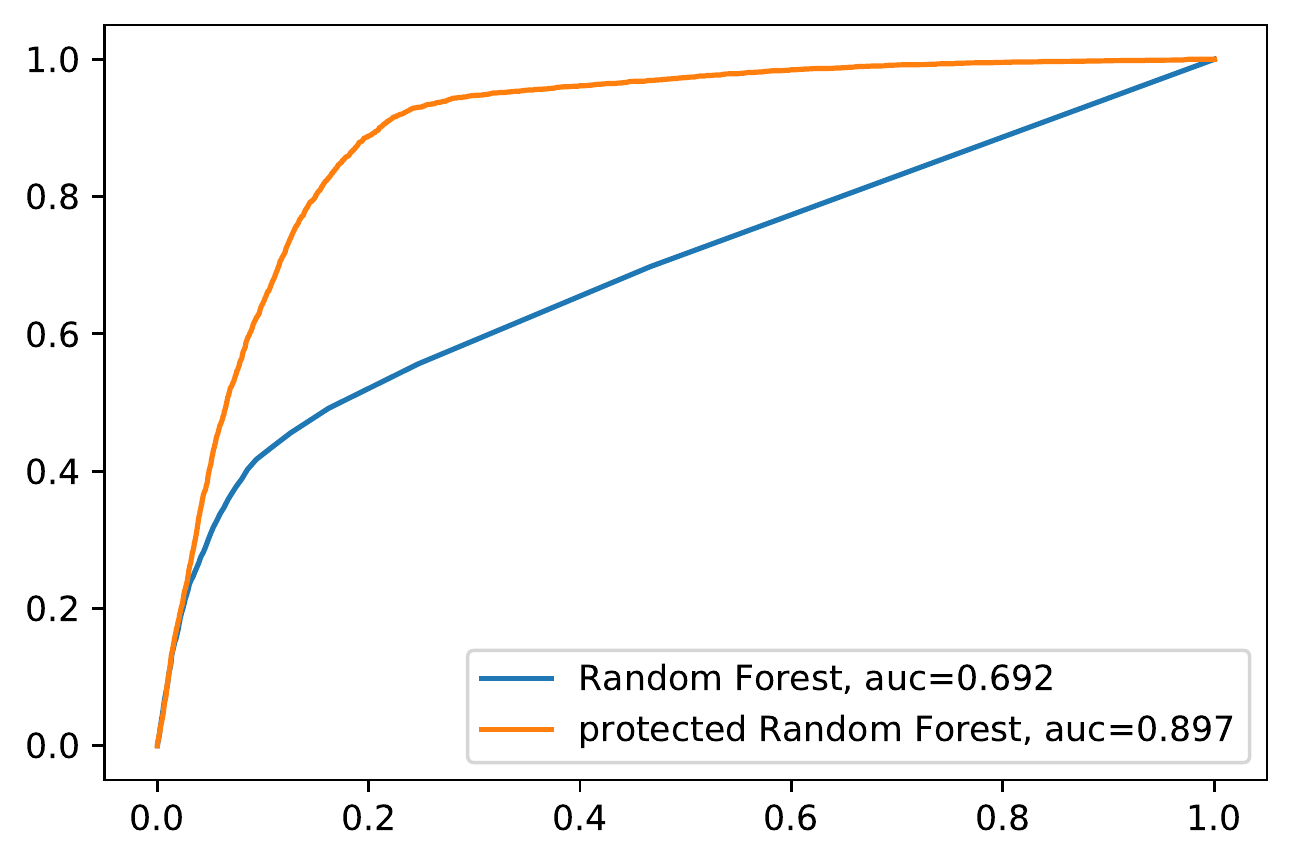}
    \includegraphics[width=0.48\textwidth]{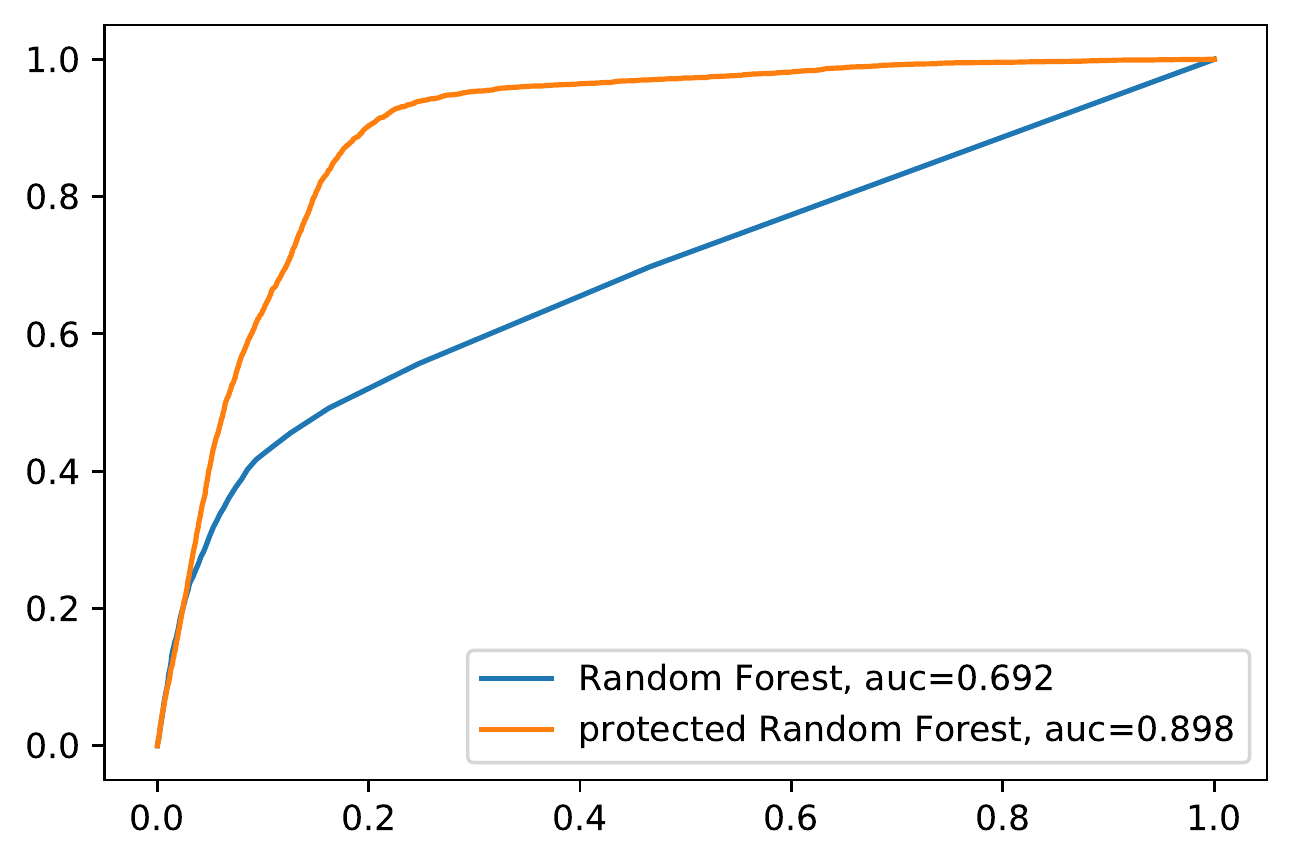}
  \end{center}
  \caption{The ROC curves for the \texttt{Bank Marketing} dataset
    and Random Forest with the Simple Jumper protection
    for the jumping rate $J:=10^{-2}$ on the left and $J:=10^{-4}$ on the right.}
  \label{fig:BM_ROC_alt}
\end{figure}

The dependence of the resulting ROC curves on the jumping rate is also weak:
see, e.g., Figure~\ref{fig:BM_ROC_alt}, which shows results for the jumping rates $0.01$ and $0.0001$.
However, using a specific value of the jumping rate $J$ may be risky in that the Simple Jumper test martingale
loses capital exponentially quickly if the base prediction algorithm is already ideal;
this will make the insurance policy discussed in Section~\ref{sec:introduction} expensive
both for testing and for prediction.
A safer option is to use the Composite Jumper procedures,
as we do in the main paper.

\begin{figure}
  \begin{center}
    \includegraphics[width=0.48\textwidth]{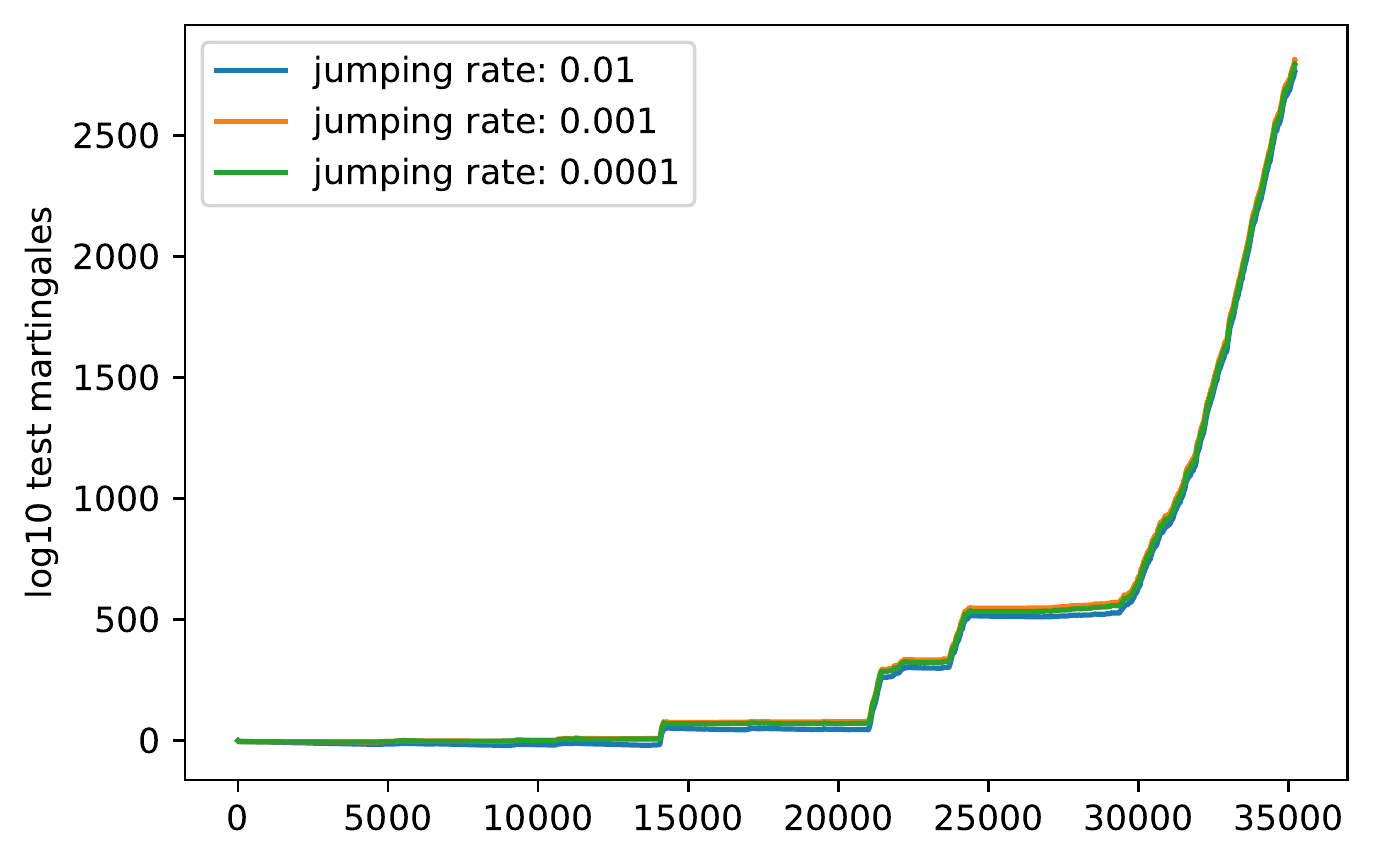}
    \includegraphics[width=0.48\textwidth]{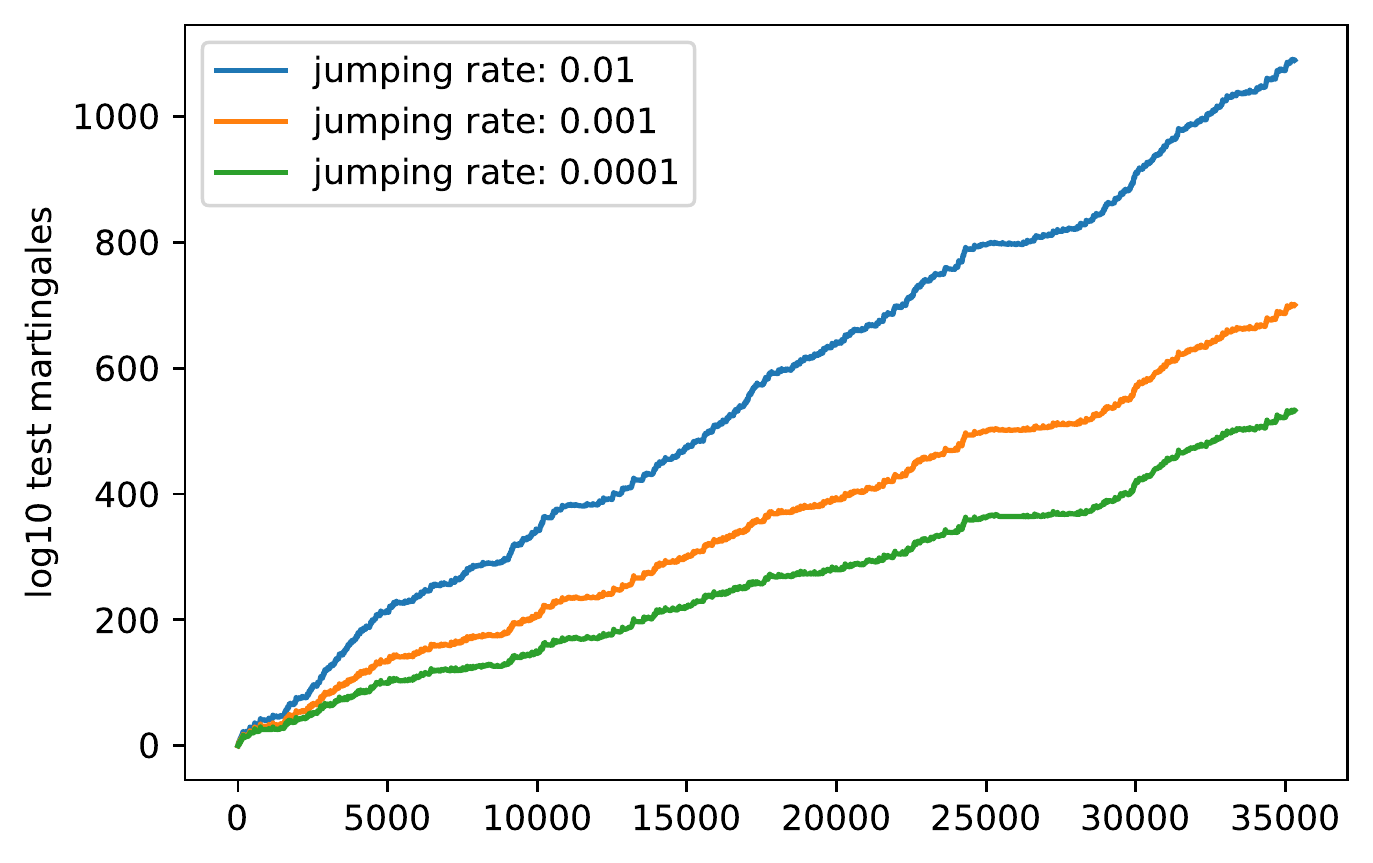}
  \end{center}
  \caption{The Simple Jumper martingales for the \texttt{Bank Marketing} and \texttt{electricity} dataset
    with randomly permuted objects.}
  \label{fig:SJs_perm}
\end{figure}

Figure~\ref{fig:SJs_perm} showing the Simple Jumper martingales for various jumping rates
is interesting because of the striking difference
between the \texttt{Bank Marketing} and \texttt{electricity} datasets
with randomly permuted objects.

\subsection*{Limited feedback}

\begin{figure}
  \begin{center}
    \includegraphics[width=0.48\textwidth]{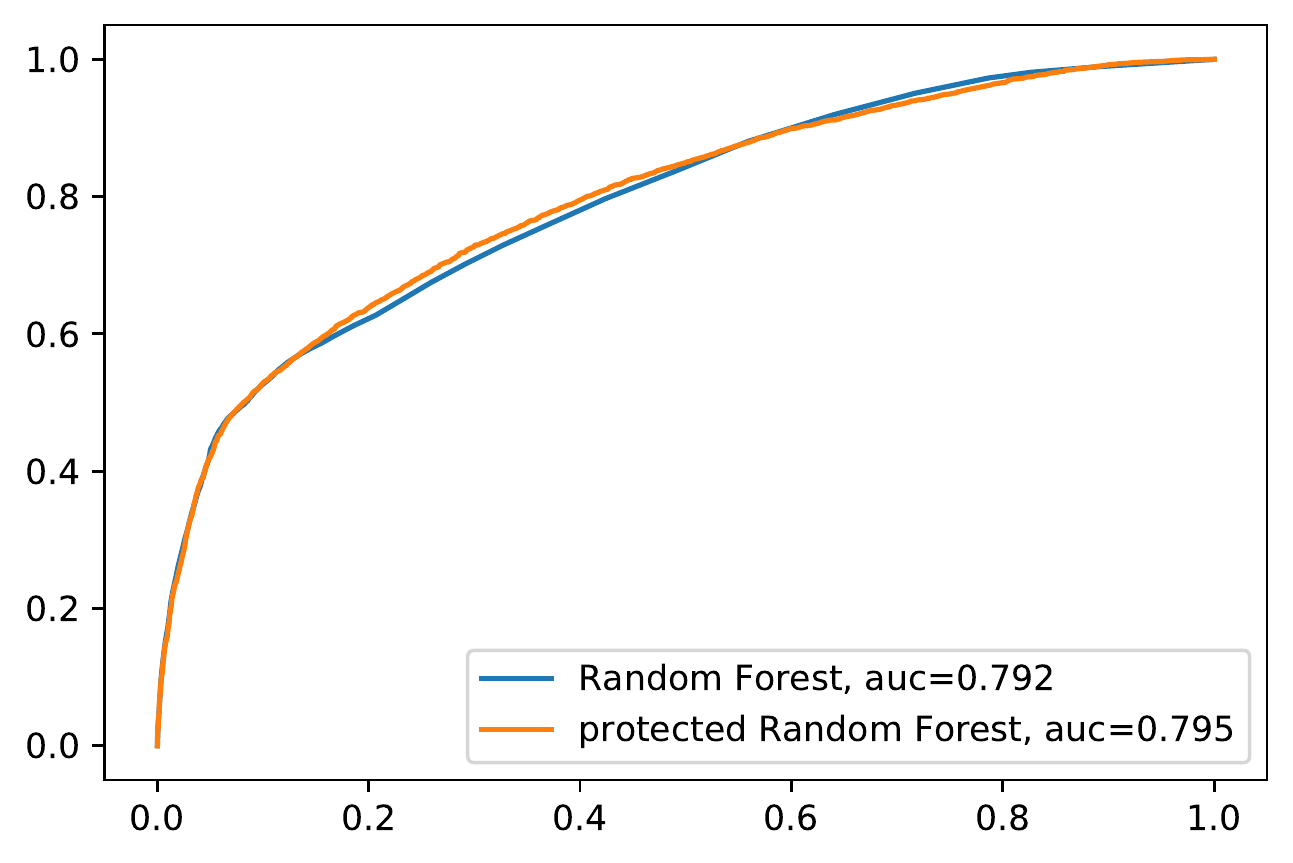}
    \includegraphics[width=0.48\textwidth]{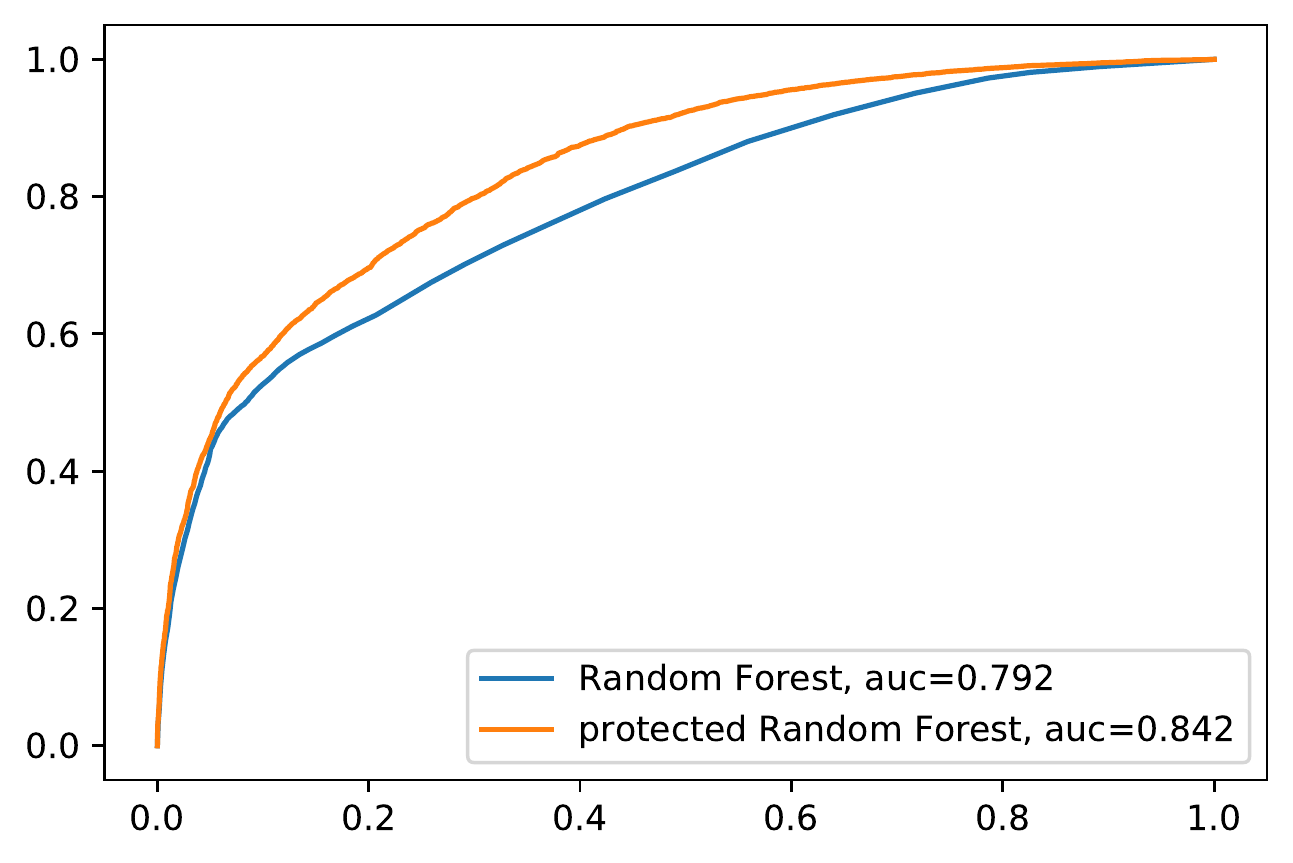}
  \end{center}
  \caption{Left panel:
    The ROC curve for the \texttt{electricity} dataset
    and Random Forest with the Composite Jumper protection
    and feedback provided for every 100th test observation.
    Right panel:
    the counterpart of the left panel for feedback provided for every 10th observations.}
  \label{fig:el_Cox_ROC_lim}
\end{figure}

The left panel of Figure~\ref{fig:el_Cox_ROC_lim}
is the counterpart for the \texttt{electricity} dataset of Figure~\ref{fig:BM_Cox_ROC_lim}:
feedback is provided only for every 100th test observation.
Now protection with limited feedback results only in marginal improvement in the AUC for the ROC curve.
With the fuller feedback comprising every 10th test observation
the improvement is again substantial: see the right panel of Figure~\ref{fig:el_Cox_ROC_lim}.

\subsection*{The time-series aspects of the \texttt{electricity} dataset}

\begin{figure}
  \begin{center}
    \includegraphics[width=0.48\textwidth]{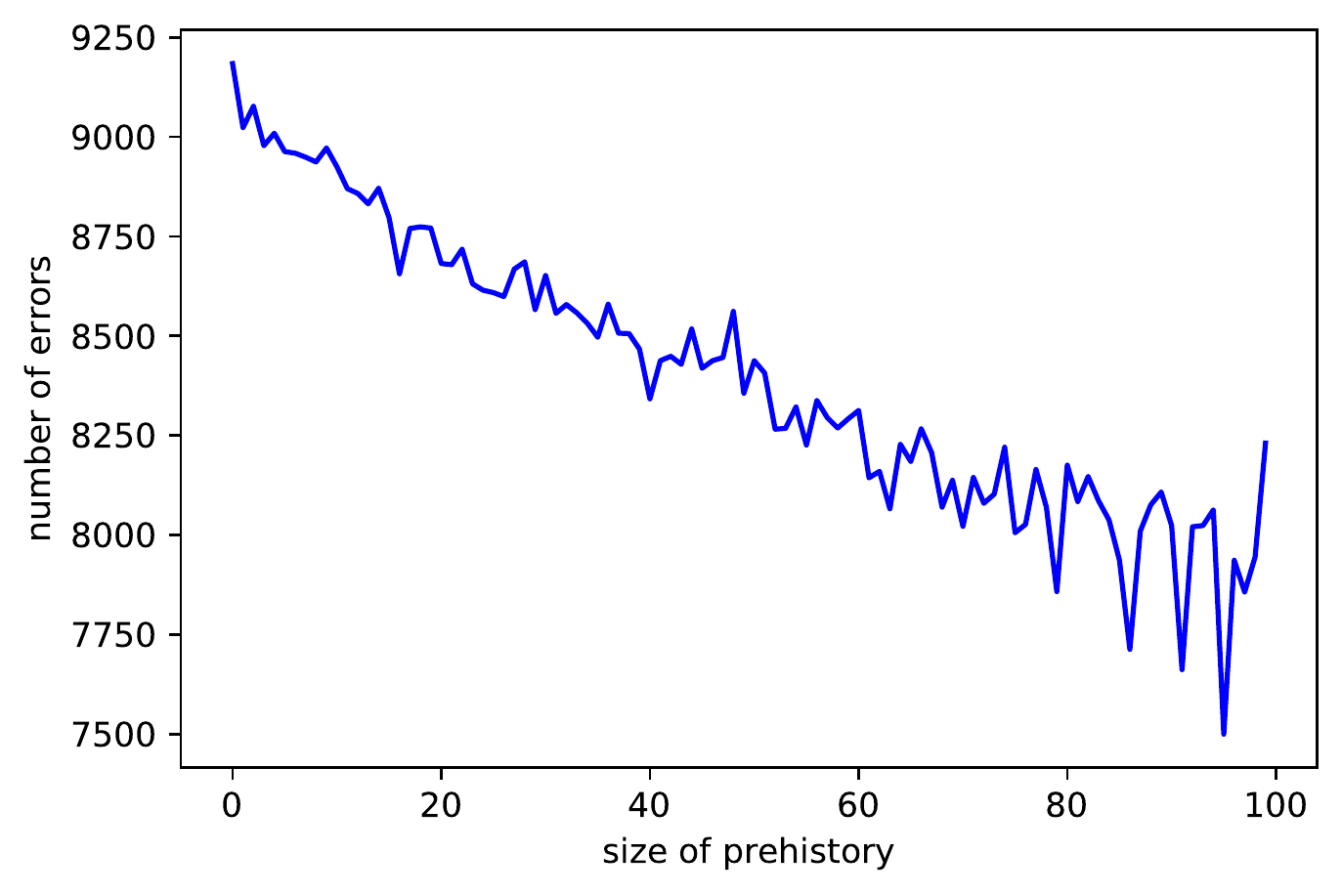}
    \includegraphics[width=0.48\textwidth]{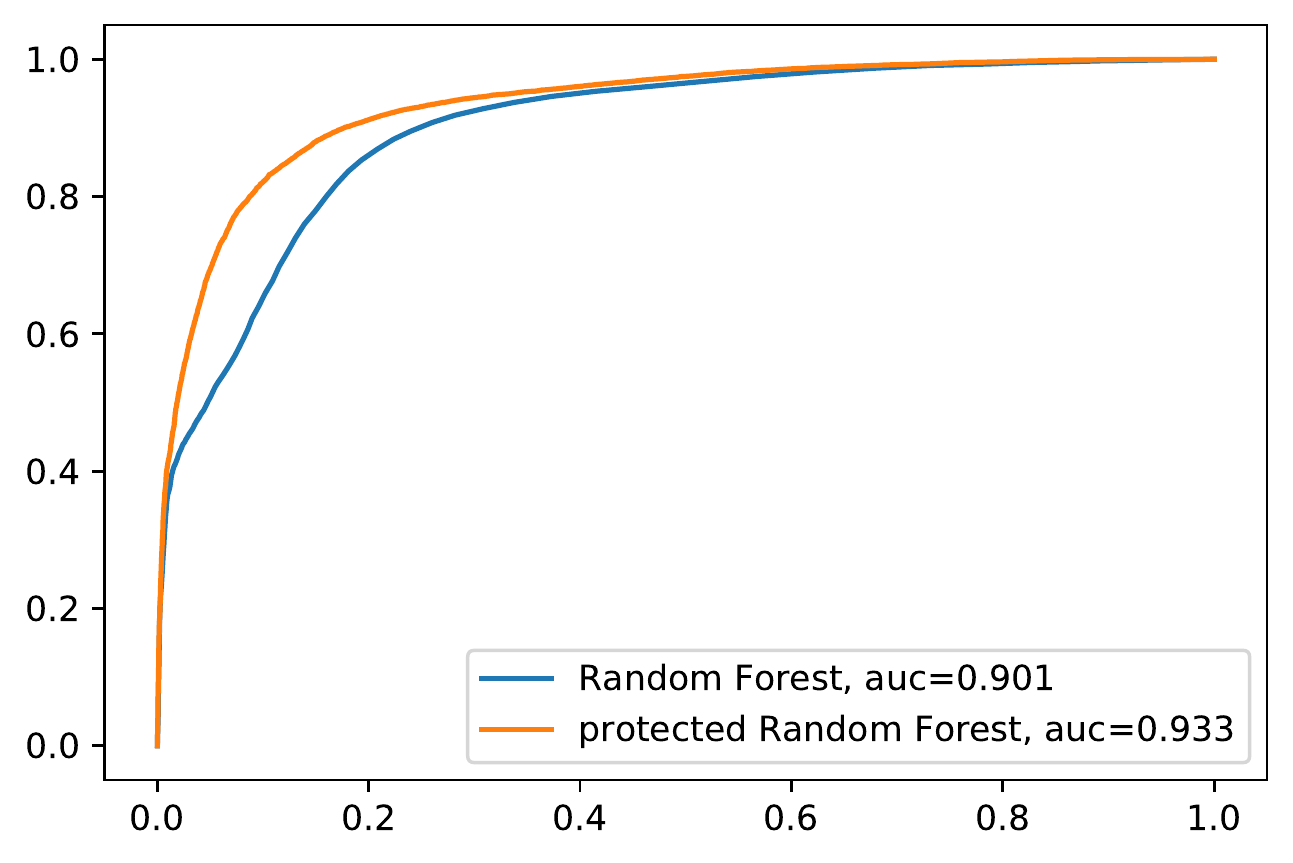}
  \end{center}
  \caption{Left panel:
    The number of errors made by Random Forest on the \texttt{electricity} dataset
    with added prehistory.
    Right panel:
    The ROC curve for the \texttt{electricity} dataset with prehistory of 48
    and Random Forest with the Composite Jumper protection.}
  \label{fig:el_prehistory}
\end{figure}

Unlike the two other datasets considered in this paper,
the \texttt{electricity} dataset consists of periodic observations
referring to a period of 30 minutes, so that there are 48 instances for each time period of one day.
Therefore, complementing the existing attributes of each observation by a prehistory,
i.e., the labels of a given number (the \emph{size of prehistory}) of immediately preceding observations,
may improve the predictions.
The left panel of Figure~\ref{fig:el_prehistory} shows that this is indeed the case.
The right panel of Figure~\ref{fig:el_prehistory}
shows the improvement in the ROC curve for the prehistory of size 48 (one day).

\subsection*{Multiclass case}

\begin{figure}
  \begin{center}
    \includegraphics[width=0.6\textwidth]{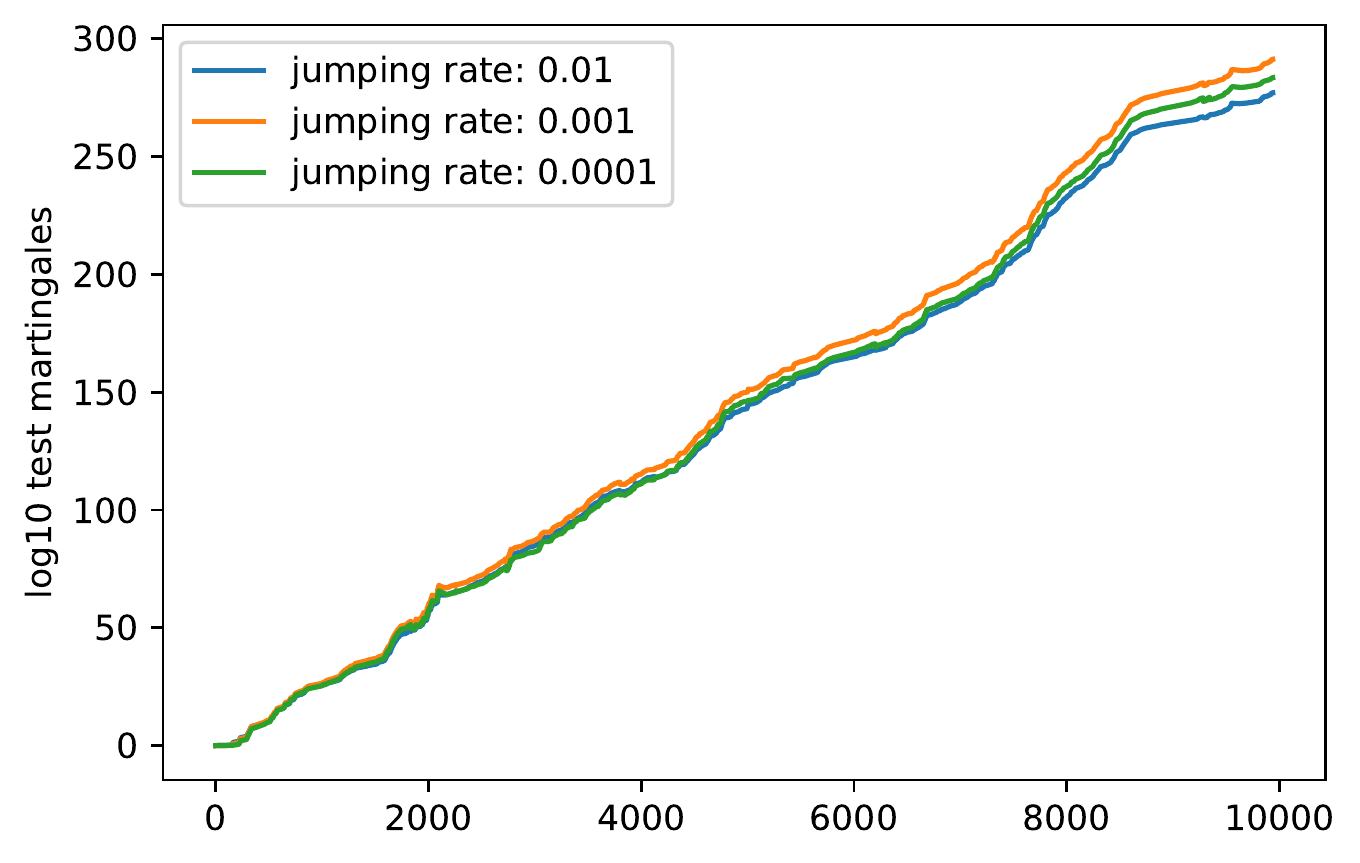}
 \end{center}
  \caption{The Simple Jumper martingales for the \texttt{UJIIndoorLoc} dataset (Scenario 1) and Random Forest.}
  \label{fig:UJI_Cox_SJs}
\end{figure}

Figure~\ref{fig:UJI_Cox_SJs} gives the trajectories of the three components of the Composite Jumper martingale
with $\mathbf{J}=\{10^{-2},10^{-3},10^{-4}\}$ and $\Theta=(\{0,1\}^3\setminus\{(1,1,1)\})\times\{0.5,1,2\}$ (as before)
based on Random Forest for the \texttt{UJIIndoorLoc} dataset in Scenario 1.
On the log scale, the three components do not appear vastly different,
but the final value for the jumping rate $0.001$ is more than $10^6$ times larger
than the final value for the jumping rate $0.01$.
The trajectory (not shown) for the Composite Jumper martingale with these jumping rates and $\pi=0.5$
is visually indistinguishable from the highest of the three trajectories that are shown
(the one for the jumping rate $0.001$).
\end{document}